\documentclass[10pt,journal,compsoc]{IEEEtran}
\ifCLASSOPTIONcompsoc
  \usepackage[nocompress]{cite}
\else
  \usepackage{cite}
\fi
\ifCLASSINFOpdf
\else
\fi

\usepackage[english]{babel}
\usepackage{amsmath}
\usepackage{amsfonts}
\usepackage{array}
\usepackage{graphicx}
\usepackage{bm}
\usepackage{fixmath}
\usepackage{blindtext}
\usepackage{amsthm}
\usepackage{mathrsfs}
\usepackage{enumerate}
\usepackage{amsthm}
\usepackage{float}\usepackage{nomencl}\usepackage{multirow}
\usepackage{bm}
\usepackage{lipsum}
\usepackage{amssymb}
\usepackage{algorithmic}
\usepackage{epstopdf}
\usepackage{caption}
\usepackage{subcaption}\usepackage{xcolor}
 \usepackage{url}\usepackage{lettrine}
 \usepackage{balance}
\usepackage{enumerate}
\usepackage{amsmath}\usepackage{amssymb}
\usepackage{amsthm}
\usepackage{algorithm}
\usepackage{graphicx}
\usepackage{float}
\usepackage{bm} 
\usepackage{epstopdf}  

\usepackage{doi}
\usepackage{xcolor}
\usepackage{mathtools}
\usepackage{multirow}
\usepackage{bm}
\usepackage{makecell}

\usepackage{mathrsfs} \usepackage{upgreek}
\usepackage{mathrsfs}
\usepackage{enumerate}
\usepackage{amsmath}
\usepackage{amsthm}
\usepackage{caption}
\usepackage{graphicx}\usepackage{lettrine}
\usepackage{float}
\usepackage{bm}\usepackage{array}
\usepackage{algorithm}
\usepackage{lipsum}
\usepackage{subcaption}
\usepackage{amssymb}
\usepackage{epstopdf} 
\usepackage{nomencl}\usepackage{stfloats}
\usepackage{url}\usepackage{booktabs}\usepackage{colortbl}\usepackage{epstopdf}
\usepackage{etoolbox}
\def\BibTeX{{\rm B\kern-.05em{\sc i\kern-.025em b}\kern-.08em
    T\kern-.1667em\lower.7ex\hbox{E}\kern-.125emX}}

\usepackage{nicefrac}       
\usepackage{microtype}      
\usepackage{lipsum}         
\usepackage{doi}
\usepackage{xcolor}
\usepackage{mathtools}
\usepackage{multirow}
\usepackage{bm}
\usepackage{makecell}

\newcommand{\defeq}{\vcentcolon=}

\usepackage{array}
\usepackage{textcomp}
\usepackage{stfloats}
\usepackage{url}
\usepackage{verbatim}

\newtheorem{theorem}{Theorem}
\newtheorem{lemma}{Lemma}
\newtheorem{definition}{Definition}
\newtheorem{remark}{Remark}

\newtheorem{proposition}{Proposition}

\title{FedCiR: Client-Invariant Representation Learning for Federated Non-IID
Features 
}
\author{\IEEEauthorblockN{Zijian Li,~\IEEEmembership{Graduate Student Member, IEEE}}, \IEEEauthorblockN{Zehong Lin,~\IEEEmembership{Member, IEEE}},\\
\IEEEauthorblockN{Jiawei Shao,~\IEEEmembership{Graduate Student Member, IEEE}},
\IEEEauthorblockN{Yuyi Mao,~\IEEEmembership{Member, IEEE}},
\IEEEauthorblockN{Jun Zhang,~\IEEEmembership{Fellow, IEEE}}
\thanks{Z. Li, Z. Lin, J. Shao, and J. Zhang are with the Department of Electronic and Computer Engineering, The Hong Kong University of Science and Technology,
Hong Kong, China (E-mail: \{zijian.li, jiawei.shao\}@connect.ust.hk, \{eezhlin, eejzhang\}@ust.hk).
Y. Mao is with the Department of Electrical and Electronic,
The Hong Kong Polytechnic University, Hong Kong, China (E-mail: yuyi-eie.mao@polyu.edu.hk) (Corresponding author: Zehong Lin.)}}

\begin{document}


\IEEEtitleabstractindextext{%
\begin{abstract}

   Federated learning (FL) is a distributed learning paradigm that maximizes the potential of data-driven models for edge devices without sharing their raw data.
   However, devices often have non-independent and identically distributed (non-IID) data, meaning their local data distributions can vary significantly. The heterogeneity in input data distributions across devices, commonly referred to as the feature shift problem, can adversely impact the training convergence and accuracy of the global model.
   To analyze the intrinsic causes of the feature shift problem, we develop a generalization error bound in FL, which motivates us to propose FedCiR, a client-invariant representation learning framework that enables clients to extract informative and client-invariant features.
   Specifically, we improve the mutual information term between representations and labels to encourage representations to carry essential classification knowledge, and diminish the mutual information term between the client set and representations conditioned on labels to promote representations of clients to be client-invariant.
    We further incorporate two regularizers into the FL framework to bound the mutual information terms with an approximate global representation distribution to compensate for the absence of the ground-truth global representation distribution, thus achieving informative and client-invariant feature extraction.
 To achieve global representation distribution approximation, we propose a data-free mechanism performed by the server without compromising privacy.
   Extensive experiments demonstrate the effectiveness of our approach in achieving client-invariant representation learning and solving the data heterogeneity issue.

\end{abstract}

 }
\maketitle


%
%
%
%
\IEEEraisesectionheading{\section{Introduction}\label{sec:introduction}}
\IEEEPARstart{W}{ith}
the ever-increasing popularity and deployment of edge devices, there is a tremendous amount of data generated by users on a daily basis. 
These data hold valuable insights that can be utilized to enhance various artificial intelligence (AI) applications and services.
However, with the increasing privacy concerns as well as the enormous communication costs, reckless data collection and centralized training are no longer acceptable.
Serving as an efficient and privacy-preserving training paradigm, federated learning (FL) enables distributed devices to collaboratively train a global model without revealing their private data.
The classical FL approach, represented by FedAvg \cite{fl}, has successfully demonstrated its potential in edge-AI tasks, such as autonomous driving \cite{auto_driving}, virtual reality (VR) services \cite{9741351}, and unmanned aerial vehicle control \cite{qi2022task}. 

One major obstacle to efficient FL is the data heterogeneity among devices, i.e., the data across devices are non-independently and identically distributed (non-IID). 
It has been verified that FL algorithms often suffer from degraded accuracy and slow convergence with non-IID data, which are inherently caused by deflected local updates and local optima \cite{scaffold,non_iid_1,non_iid}.
The non-IID issue in FL can be classified into two categories: label shift and feature shift \cite{kairouz2021advances}.
Label shift arises when clients have different distributions of class labels, even if the input feature space is the same. Feature shift, on the other hand, occurs when the underlying distributions of input features vary across clients, despite having a common label space.
In this work, we focus on the \emph{feature shift} problem \cite{fedbn}, which typically arises in healthcare \cite{harmoFL} and wildlife protection applications \cite{wilds} because of the device and environmental discrepancy.
For example, in wildlife protection applications where the cameras are placed in different locations to capture images of wildlife, the images from these cameras are different in tonality, rendering them skewed in features.
These discrepancies are primarily caused by distinct environmental conditions like weather and illumination levels.

\begin{figure}
    \centering \includegraphics[width=0.48\textwidth]{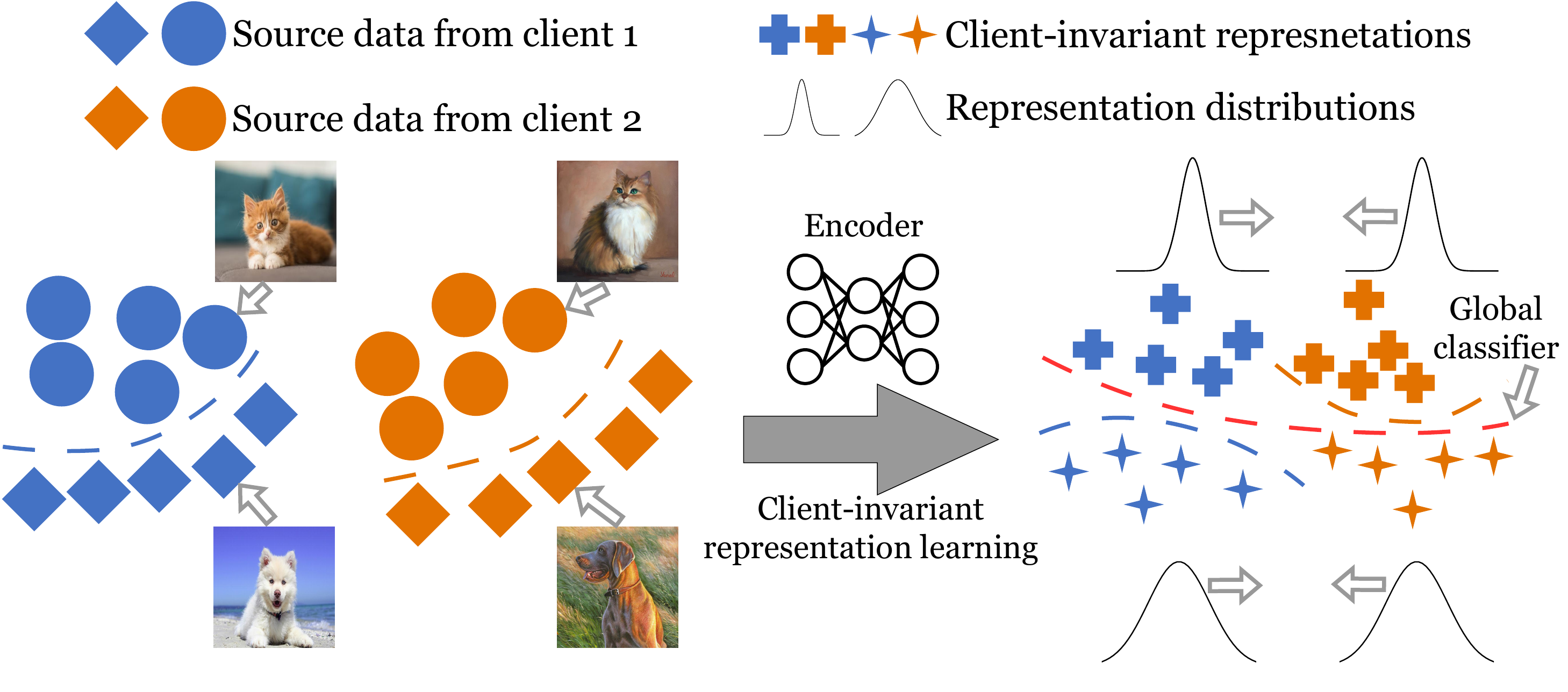}
    \caption{Overview of the proposed method with two clients. Clients extract client-invariant features from feature-skewed samples so that they can be successfully discriminated by the classifier.}
    \label{fig:idea}
\end{figure}

To solve the feature shift problem in FL, many studies proposed to introduce local constraints with the knowledge of the global model or gradient information for clients, whereby the discrepancy between local updates is mitigated.
These methods, however, focused on the algorithmic aspect and only achieved limited performance gains with heavy data heterogeneity \cite{fl_experiments}.
Meanwhile, another stream of studies endeavored to solve the feature-shift issue through synthetic data augmentation \cite{fedgen, fedgan,feddpgan, fmds_fl}, but sharing extra synthetic data among clients inevitably leads to data privacy leakage.
In addition, explicit feature alignment among clients is another effective technique to solve the feature shift problem \cite{fedproc, fedfa}.
By sharing and aggregating local features as global features, clients can align the local features with the global features, thereby reaching a common feature space for solving the feature shift problem.
Nevertheless, the local features remain susceptible to overfitting with heterogeneous data distributions, resulting in biased global features after aggregation.

The feature shift problem has been extensively studied in classical centralized training, particularly in the context of domain adaptation and domain generalization \cite{domain_adaptation_1,domain_adaptation_2}, which aim to transfer the knowledge of labeled source domains to a different target domain.
Many studies have established the generalization error bounds of the target domain on the source domains, using the source risk and the discrepancy between source and target domains. 
These bounds provide sufficient conditions to reduce the error on the domain dataset and thus address the centralized feature-skewed issue.
Consequently, many works \cite{fea_ali_1,fea_ali_5} have employed domain-invariant representation learning to reduce the generalization bound for better domain adaptation and generalization. 
However, the bounds derived in centralized settings are incompatible with FL since the source domains are distributed across clients.
With private distributed data, all the existing domain-invariant representation methods for domain adaptation lose their effectiveness without the access to all the data domains during the training process.
Therefore, it is critical to analyze the intrinsic causes of the feature-skewed issue in FL and propose effective methods to address it accordingly.

In this work, we first develop a generalization error bound for the global mixture dataset on local datasets in FL, which identifies two key factors for addressing the feature shift issue: the empirical local risks and the $\mathcal{H}$-divergence between global and local datasets.
Motivated by this analysis, we propose a \emph{federated} \emph{client-invariant} \emph{representation learning} algorithm, called \emph{FedCiR}, to reduce the generalization error bound by encouraging local models to extract informative and client-invariant features from an information-theoretic perspective.
Specifically, FedCiR introduces two mutual information terms. The first term promotes features to be informative by maximizing the mutual information between representations and labels, which addresses the information-insufficient issue in local models caused by the limited data. The second term diminishes the mutual information between the client set and representations conditioned on labels, which advances representations across clients to be client-invariant. To integrate these mutual information objectives into FL, we develop two regularizers that bound the objectives using an approximate global representation distribution. Such an approximation is needed since the ground truth distribution is unknown in FL. To achieve this approximation without compromising privacy, we propose a novel data-free mechanism using a generator, which is trained to maximize the consensus between local classifiers at the server. The first proposed regularization enhances the global view for local client representations, while the second explicitly aligns global and local distributions, ensuring representations retain just the client-invariant information needed.
Figure \ref{fig:idea} depicts the main idea of the proposed method.
The main contributions of this work are summarized as follows:
\begin{itemize}
    \item To provide a theoretical foundation for addressing the non-IID issue in FL, we establish a generalization error bound between the global and local datasets, which motivates us to propose a federated client-invariant representation learning algorithm, called FedCiR.
    \item We propose two mutual information terms to achieve informative and client-invariant representation extraction, respectively, and introduce two regularization terms based on an approximate global representation distribution to bound the mutual information terms, which can be employed in FL and encourages local representations to capture more label-discriminative and client-invariant information.
    \item We propose a data-free mechanism to approximate the global representation distribution by using a generative network to ensure the consistency of local classifiers at the server, which protects the privacy as strict as FedAvg.
    \item We conduct extensive experiments on benchmark and real-world datasets, which demonstrate that FedCiR consistently outperforms baselines in accuracy without introducing additional overhead for clients, and confirm its effectiveness in client-invariant representation learning.
\end{itemize}

\textbf{Organization:} The rest of this paper is organized as follows. In Section \ref{related_work}, we review the related works on the non-IID problem in FL and domain-invariant representation learning. Section
\ref{Preliminary} describes the feature shift problem in FL, and Section \ref{generalization_bound} illustrates the fundamental causes of the feature shift problem with a federated generalization error bound.
We propose a class-discriminative and client-invariant representation learning method for FL (FedCiR) and a data-free mechanism for global representation distribution approximation to achieve FedCiR in Section \ref{method}. 
We evaluate FedCiR via extensive simulations in Section \ref{experiments}. 
Finally, we conclude the paper in Section \ref{conclusions}.

\section{Related Works}
\label{related_work}
\textbf{FL under data heterogeneity:}
FL was originally proposed by McMahan et al. \cite{fl} as a distributed training paradigm to train a global model without disclosing local private data.
However, the statistical heterogeneity, which adversely affects the convergence and accuracy performance \cite{non_iid_1,non_iid}, quickly emerged as a major obstacle in FL.
To address this issue, many works added proximal terms to regularize local objective functions with the global knowledge \cite{feddyn,scaffold,moon,fedprox}.
For instance, FedProx \cite{fedprox} included a regularizer into local training to force the local updates to get close to the global model.
FedDyn \cite{feddyn} proposed a dynamic regularizer for each client to ensure the consistency between the stationary points of local and global empirical loss.
In addition, MOON \cite{moon} utilized contrastive learning to reduce the distance between representations output by the global and local models to correct the local training.
These methods, however, cannot fundamentally eliminate the data heterogeneity and may encounter the performance bottleneck in extreme cases with highly heterogeneous data distributions \cite{fl_experiments}.

Moreover, instead of pushing local parameters to trace the global parameters, some methods \cite{fedproc,fedproto,fedfa} attempted to align the intermediate features between clients through the global prototypes, which are computed by averaging the features for each class locally and aggregating them among clients at the server.
However, the local prototypes suffer from overfitting problems with data heterogeneity, and naively aggregating them leads to imprecise global prototypes.
In contrast, this work aims to align the representation distributions of clients with a global representation distribution, which is approximated by maximizing the consensus among local classifiers due to its absence in FL, thus improving training stability and efficiency.

Another research line for data heterogeneity relies on data augmentation \cite{non_iid_survey, what2share}. 
Many studies attempted to generate and share synthetic data using a generative model \cite{sda_fl,fedgan,feddpgan} or a mixup mechanism \cite{fedmix}.
Nevertheless, recklessly sharing synthetic samples raises extra privacy concerns.
To mitigate potential privacy leakage, recent approaches have developed data-free knowledge distillation mechanisms to generate auxiliary data with a generative model at the server, which can be used for global model calibration \cite{fedftg} and local model updating \cite{VHL,fedgen}.
These methods, however, required clients to upload the training data distributions to the server to optimize the feature generator, which is excessively vulnerable to gradient inversion attacks \cite{gradient_inversion}. Different from them, we utilize the ensemble predictions of local classifiers to train the feature generator without extra local information leakage.
More importantly, we further utilize the generator to estimate the global representation distribution for client-invariant representation extraction, which benefits solving the feature shift problem. 

\textbf{Domain-invariant representation learning:}
\label{domain_invariant}
Domain-invariant representation learning has been proven effective for domain adaptation and generalization by reducing the discrepancy between source and target domains \cite{theory_da}.
Adversarial learning is a widely used mechanism to learn the domain-invariant features by gradually disentangling the irrelevant and domain-specific features via additional discriminators \cite{fea_ali_1,adver_3,adver_4}.
The adversarial networks are able to push representation distributions of different domains to be the same, thereby enabling domain-invariant representation learning. 
FedADA \cite{fedada} and FedADG \cite{fedadg} adopted this technique to effectively solve domain adaptation and domain generalization problems, respectively, in FL.
However, due to the limited domains of data in each device, they demanded to train additional generators to generate the lacking domains of data for adversarial training, which aggravates the computation burden for devices.

Another stream of methods for domain-invariant representation learning is explicit feature distribution alignment \cite{fea_ali_4,fea_ali_5,DG_review}, which can be achieved by minimizing the distance between source and target distributions \cite{fea_ali_mmd,second_order_correlation,domainnet}.
However, these methods require the access to all the domains of data during the training process and thus are incompatible with FL. 
A recent work FedSR \cite{fedsr} used probabilistic networks to approximate feature distributions for clients and aligned them via iterative local updating and global aggregation.
Nevertheless, FedSR is sensitive to the client number and suffers from the local overfitting problem due to the distinct divergence of deterministic networks between clients caused by the non-IID data.
Instead, this work approximates the global representation distribution with the ensemble predictions of local models at the server, thus improving the stability and effectiveness of feature distribution alignment.

\section{System Model and Problem Setup}
\label{Preliminary}


Let $\mathcal{X}$, $\mathcal{Z}$, and $\mathcal{Y}$ represent the input, representation, and output spaces, respectively.
We use $X$, $Z$, and $Y$ to denote random variables that are sampled from $\mathcal{X}$, $\mathcal{Z}$, and $\mathcal{Y}$, respectively. 
We adopt $\bm{x},y$, and $\bm{z}$ to represent the corresponding instances of random variables.
In the FL setting with the client set $\mathcal{K}$, each client $k \in \mathcal{K}$ has its own distribution $\mathcal{D}_k=p(\bm{x},y|k)$, and we denote the joint distribution among clients by $p(k,\bm{x},{y})=p(k)p(y|k)p({\bm{x}}|k,y)$. 
The goal of FL is to train a global model $\bm{w}$ that minimizes the expected risks for clients as follows:
\begin{align}
    \mathcal{L}_g (\bm{w}) = \mathbb{E}_{({\bm{x}},y)\sim \Tilde{\mathcal{D}}} \ell (f(\bm{w};\bm{x}), y),
\end{align}
where $f: \mathcal{X} \rightarrow \mathcal{Y} $ is the function that maps the input data $\bm{x}$ to a predicted label $\hat{y}$ using the model $\bm{w}$, $\ell$ is a generic loss function, e.g., the cross-entropy loss, and $\Tilde{\mathcal{D}}=\cup \{ \mathcal{D}_k\}_{k\in \mathcal{K}} $ is the global mixture distribution of all the clients.
To achieve this goal, the global objective function for FL is given by:
\begin{equation}
    \mathcal{L} (\bm{w}) =  \sum _{k \in \mathcal{K}} p(k) \mathcal{L}_k(\bm{w}),
\end{equation}
where
\begin{align}
    \mathcal{L}_k (\bm{w}) = \mathbb{E}_{(\bm{x},y)\sim \mathcal{D}_k} \ell ( f(\bm{w};\bm{x}), y)
\end{align}
is the local objective function of client $k$.

In the classical FedAvg method \cite{fl}, the optimization of the model $\bm{w}$ is achieved through a collaborative effort between clients over $\mathcal{T}$ communication rounds. 
In each communication round $t\in [\mathcal{T}]$, a subset of clients $\mathcal{K}_t \subseteq \mathcal{K}$ are randomly selected by the server and download the global model $\bm{w}_t$ for local training.
Due to the impracticality of calculating full gradients on the entire local dataset $\mathcal{D}_k$, each active client $k \in \mathcal{K}_t$ performs local model updates with $E$ local steps using the mini-batch stochastic gradient descent (SGD).
To be specific, in each local step $e \in [E]$ of the $t$-th communication round, client $k$ randomly samples a batch of training data $\xi_{k}^{t,e}$ and updates local model $\bm{w}_{k}^{t,e-1}$ with the gradients $\nabla \ell (\bm{w}_k^{t, {e-1}}; \xi_k^{t,e})$ as follows:
\begin{align}
    \bm{w}_{k}^{t,e} \leftarrow \bm{w}_{k}^{t,e-1} - \eta\nabla \ell (\bm{w}_k^{t, {e-1}}; \xi_k^{t,e}),
\end{align}
where $\eta$ is the learning rate.
After the $E$-step local training, each active client $k \in \mathcal{K}_t$ sends its local model $\bm{w}_k^t$ to the server for global model aggregation as follows:
\begin{align}
    \bm{w}^{t+1} = \frac{1}{|\mathcal{K}_t|} \sum_{k \in \mathcal{K}_t} \bm{w}^{t}_{k}, \quad \bm{w^}t_k \leftarrow \bm{w}_{k}^{t,E}.
\end{align}
Afterwards, the new global model $\bm{w}^{t+1}$ is broadcast to the newly active clients for the next round of training.

Despite the success of FL in many practical applications, the heterogeneous data distributions, including label-skewed and feature-skewed distributions, still compromise the training performance in terms of convergence and accuracy \cite{fl_experiments,non_iid}. 
In this work, we aim at addressing the federated feature-skewed problem, where the clients have varying input distributions (domains) for a given label class $y_c$, i.e., $p(\bm{x}|k, y_c) \neq p(\bm{x}|k^\prime, y_c)$, where $k,k^\prime \in \mathcal{K}$ and $k \neq k^\prime$. 
Feature-skewed problems often arise in edge AI applications due to the variations in natural conditions and processing methodologies across edge devices.
For example, the images taken by smartphones in a rural area may be different from those taken by smartphones in an urban area.
With the feature distribution shift among clients, traditional FL methods, like FedAvg, attempt to learn all possible feature knowledge, including client-invariant and client-specific knowledge, during local training. 
However, the client-specific knowledge causes representation divergence, i.e., $p(\bm{z}|k,y,\bm{x})\neq p(\bm{z}|k^\prime,y,\bm{x})$, leading to an accuracy reduction.
Although many studies have proposed to alleviate the feature shift issue by reducing the divergence of local models \cite{non_iid_survey}, guidelines that can fundamentally solve this problem still remain unclear.

\section{Generalization Error Bound for Federated Learning}
\label{generalization_bound}
To analyze the fundamental causes of the federated feature-skewed problem, in this section, we first introduce a generalization error bound in domain adaptation \cite{domain_adaptation_2},
which is able to solve the feature shift problem in centralized training by minimizing the generalization error on the target domain.
Afterwards, we extend this bound for FL, which motivates us to develop a label-discriminative and client-invariant representation learning method in Section \ref{method}.

To ensure clarity and enhance understanding, following \cite{domain_adaptation_2}, we supplement some relevant definitions before delving into the analysis.

\begin{definition}
    (\textit{Hypothesis ${\bm{h}}$})
    Let $\mathcal{V}$ be a one-hot vector set that contains $N$ one-hot vectors for an $N$-class classification task.
    Given a sample $\bm{x} \in \mathcal{X}$, a hypothesis $\bm{h}: \mathcal{X} \rightarrow \mathcal{V}$ outputs the one-hot vector $\bm{v} \in \mathcal{V}$ with $\bm{v}_c=1$ and $\bm{v}_j=0, j \neq c$ to represent the class label $c \in [N]$. 
\end{definition}

\begin{definition}
    (\textit{Domain $\langle \mathcal{D},\bm{r} \rangle$})
    We define $\langle \mathcal{D},\bm{r} \rangle$ as a domain, where $\mathcal{D}$ is a data distribution over $\mathcal{X}$ and $\bm{r}: \mathcal{X} \rightarrow \mathcal{V}$ represents its ground-truth labeling function.
\end{definition}

\begin{definition}
    (\textit{Conditional function class $\mathcal{G}$} on $\mathcal{H}$) For a function class ${\mathcal{H}} \subseteq [0,1]^\mathcal{X}$ on the input space $\mathcal{X}$, we define a function class $\mathcal{G}$ conditioned on $\mathcal{H}$ for the hypotheses $\bm{g}$, and $\bm{g} (\bm{x}) \coloneqq \text{sign} (| \bm{h} (\bm{x}) - \bm{h}^\prime (\bm{x})| -m )$ with ${\bm{h}}, {\bm{h}}^\prime \in {\mathcal{H}}$ and $0 \leq m \leq 1$. 
\end{definition}

\begin{definition}
    (\textit{$\mathcal{G}$-\text{divergence}}) Given the function class $\mathcal{G}$ and any two distributions $\mathcal{D}$ and $\mathcal{D}^\prime$ over $\mathcal{X}$, the $\mathcal{G}$-\text{divergence} between the distributions of $\mathcal{D}$ and $\mathcal{D}^\prime$ is defined as $d_{\mathcal{G}} (\mathcal{D}, \mathcal{D}^\prime)=\sup_{\bm{g} \in \mathcal{G}} | \Pr_{\mathcal{D}} [\bm{g} (\bm{x})=1] -   \Pr_{\mathcal{D}^\prime} [\bm{g}(\bm{x})=1] |$.
\end{definition}

With the above definitions, we obtain the generalization error bound in domain adaptation in the following theorem. 

\begin{theorem}
\textbf{(Generalization error bound in domain adaptation \cite{domain_adaptation_2})}
\label{theorem_1}
Let $\langle \mathcal{D}_S, \bm{r}_S \rangle$ and $\langle \mathcal{D}_T, \bm{r}_T \rangle$ be the source and target domains, respectively.
For a function class $\mathcal{H} \subseteq [0,1]^\mathcal{X}$ and $\forall \bm{h} \in \mathcal{H}$, the following inequality holds:
\begin{align}
    \varepsilon_{{T}} (\bm{h}) \leq &  \varepsilon_{{S}} (\bm{h}) + d_{\mathcal{G}} (\mathcal{D}_S, \mathcal{D}_T) \nonumber\\
    &+ \min \{ \mathbb{E}_{\mathcal{D}_S} [|\bm{r}_S - \bm{r}_T|], \mathbb{E}_{\mathcal{D}_T} [|\bm{r}_S - \bm{r}_T|] \}.
\end{align}
\begin{proof}
    See \cite[Theorem 4.1]{domain_adaptation_2}.
\end{proof}
\end{theorem}
The error bound in Theorem \ref{theorem_1} reveals the intrinsic causes of the feature-skewed issue in terms of the source risk, distribution discrepancy between the source and target domains, and labeling function distance between the source and target domains in domain adaptation. 
To reduce the discrepancy between source and target domains, some studies \cite{fea_ali_5,fea_ali_mmd} proposed to learn an encoder $\bm{w}_1: \mathcal{X}\rightarrow \mathcal{Z}$ to map the source distribution $\mathcal{D}_S$ and target distribution $\mathcal{D}_T$ to a common representation distribution $\mathcal{R}$ over $\mathcal{Z}$ such that the induced representations of source and target domains are closer and thereby better classified by a new hypothesis $\hat{\bm{h}}: \mathcal{Z} \rightarrow \mathcal{V}$.

However, in FL, the target distribution refers to the global mixture distribution among all the clients, i.e., $\Tilde{\mathcal{D}} = \mathcal{D}_T =\cup \{ \mathcal{D}_{k}\}_{k\in \mathcal{K}}$, which leads to a global representation distribution $\Tilde{\mathcal{R}} =\cup \{ \mathcal{R}_{k}\}_{k\in \mathcal{K}}$.
Following prior works \cite{fea_ali_4,domain_adaptation_2,covariate_shift}, we assume that the data distribution among clients is \emph{covariate shift}, i.e., the labeling functions among clients are the same, leading to a global labeling function $\Tilde{\bm{r}} \defeq  \bm{r}_k = \bm{r}_{k^\prime}$, where $ k, k^\prime \in \mathcal{K}$ and $k \neq k^\prime$.
This is reasonable because, for example, a dog image is annotated as ``dog" according to the same criteria based on its shape and outline, even if at different clients. 
Furthermore, the source domains are distributed across clients 
and cannot be shared with each other.
To establish the connection between the source and target representations w.r.t. the hypothesis $\hat{\bm{h}}$, we derive an error bound for FL in the following theorem.
\begin{theorem}
\textbf{(Federated generalization error bound)}
\label{theorem_2}
Let $\langle \mathcal{D}_{k}, \bm{r}_{k} \rangle$ and $\langle \Tilde{\mathcal{D}}, \Tilde{\bm{r}} \rangle$ be the domain of client $k$ and global domain in the FL setting, respectively.
For any convex function class $\hat{\mathcal{H}} \subseteq [0,1]^\mathcal{Z}$, convex risk function $\varepsilon (\cdot)$, and $\forall \hat{\bm{h}} \in \hat{\mathcal{H}}$, the following inequality holds:
\begin{align}
    &\varepsilon (\hat{\bm{h}}) 
    \leq \sum_{k\in \mathcal{K}} p(k) \varepsilon_{{k}} (\hat{\bm{h}}_k) + \sum_{k\in \mathcal{K}} p(k) d_{\hat{\mathcal{G}}} (\mathcal{R}_{k}, \Tilde{\mathcal{R}}),
\end{align}
where $\hat{\mathcal{G}}$ is defined as the function class conditioned on $\hat{\mathcal{H}}$ for the hypotheses $\hat{\bm{g}}$ that $\hat{\bm{g}} (\bm{z}) \coloneqq \text{sign} (| \hat{\bm{h}} (\bm{z}) - \hat{\bm{h}}^\prime (\bm{z})| -m )$ with ${\hat{\bm{h}}}, {\hat{\bm{h}}}^\prime \in {\hat{\mathcal{H}}}$ and $0 \leq m \leq 1$, and $\mathcal{R}_{k}$ and $\Tilde{\mathcal{R}}$ are the induced representation distributions of $\mathcal{D}_{k}$ and $\Tilde{\mathcal{D}}$, respectively.
\begin{proof}
    \text{See Appendix \ref{theorem_2_proof}.}
\end{proof}
\end{theorem}
Theorem \ref{theorem_2} holds when the hypothesis $\hat{\bm{h}}$ and loss function $\varepsilon (\cdot)$ are both convex.
In general, for classification tasks in FL, we can adopt the last fully-connected layer followed by a softmax layer as the classifier and the cross-entropy loss as the loss function, which are both convex, such that Theorem \ref{theorem_2} is valid.
Compared with the error bounds in \cite{fedada} and \cite{fedgen} that require the impractical joint error of the optimal hypothesis w.r.t. source and target domains, the generalization error bound derived in Theorem \ref{theorem_2} characterizes the sufficient conditions to relieve the feature-skewed issue in FL. Specifically, the first term represents the expected source errors of clients, and the second term quantifies the marginal representation distribution distance between source and target domains.

According to Theorem \ref{theorem_2}, we can reduce the risks of the global dataset by minimizing the local source errors and marginal distribution distance between global and local datasets, thereby solving the feature-skewed issue.
To achieve this goal, we propose a novel FL algorithm that learns label-discriminative and client-invariant representations in the next section.
The client-invariant representation learning is effective since it achieves smaller divergence between the global and local datasets in the feature space (the second term in Theorem \ref{theorem_2}), which further reduces the discrepancy between global and local models and achieves lower local risks (the first term in Theorem \ref{theorem_2}) in local training. 
This will be verified by experiments in Section \ref{insight}.

\section{Proposed Method}
\label{method}
In this section, we propose a method called \textit{FedCiR} to achieve label-discriminative and client-invariant representation extraction and reduce the generalization error for FL.
Specifically, we propose two mutual information terms to promote the representations to be class-discriminative and client-invariant, respectively.
Moreover, two local regularizers are incorporated into FL to bound these two mutual information terms with an approximate global representation distribution.
To achieve the approximation, we further develop a data-free mechanism by maximizing the consensus of the uploaded classifiers, which protects the privacy as strict as FedAvg.

\subsection{Class-Discriminative Representation Learning for Federated Learning}
In the proposed representation learning framework, we aim to learn the representation $\bm{z}$ of a data sample $\bm{x}$ with a distribution $p_{\bm{w}_1}(\bm{z}|\bm{x})$ parameterized by an encoder network $\bm{w}_1$ (we omit $\bm{w}_1$ afterwards for notation simplicity). 
With the usage of batch normalization (BN) layers, the representation distribution $p(\bm{z}|\bm{x})$ is assumed to be a normal distribution, which can be computed through a probabilistic mapping: $p(\bm{z}|\bm{x})=\mathcal{N}(\bm{z};\bm{\mu}_{\bm{w}_1}(\bm{x}),\bm{\sigma}_{\bm{w}_1}(\bm{x}))$, where $\bm{\mu}_{\bm{w}_1}(\cdot)$ and $\bm{\sigma}_{\bm{w}_1} (\cdot)$ are the mean and standard deviation matrices output by the encoder network $\bm{w}_1$.
By using the reparameterization trick \cite{para_trick}, we obtain $p(\bm{z}|\bm{x})d\bm{z}=p(\bm{\epsilon})d \bm{\epsilon}$, where $\bm{z}=f(\bm{x},\bm{\epsilon})$ is a deterministic function of $\bm{x}$ and $\bm{\epsilon} \sim \mathcal{N}(0,\bm{I})$ is a Gaussian random variable.
As a result, the noise term $\bm{\epsilon}$ becomes independent of the parameters $\bm{w}_1$, which makes it convenient to compute gradients and update parameters.

During the FL training process, we sample the representation $\bm{z}$ for a given $\bm{x}$ from the distribution $p(\bm{z}|\bm{x})$.
After obtaining the representation $\bm{z}$, we train a classifier $\bm{w}_2$ that predicts the label $y$ with a predictive distribution $\hat{p}_{\bm{w}_2} (y|\bm{z})$ (we also omit the parameters $\bm{w}_2$ afterwards). For general classification tasks, the predictive distribution is the output of a softmax layer that follows a linear layer.
Therefore, the predictive distribution of $y$ for $\bm{x}$ is given by:
\begin{equation}
    \hat{p}(y|\bm{x}) = \mathbb{E}_{p(\bm{z}|\bm{x})} \hat{p} (y|\bm{z}).
\end{equation}
The loss function of sample $\bm{x}$ and label $y$ is set as the negative log predictive density, i.e., $\ell (\bm{w};\bm{x},y)=- \log \mathbb{E}_{p(\bm{z}|\bm{x})} \hat{p} (y|\bm{z})$, where $\bm{w}=\bm{w}_1 \circ \bm{w}_2$ is the global model.
With the local data distribution $p(\bm{x},y|k)$ and prediction distribution $\hat{p}({y}|k,\bm{z})$, the local objective function of client $k$ is given by:
\begin{equation}
\label{local_obj}
    \mathcal{L}_k = \mathbb{E}_{p(\bm{x},y|k)} [- \log \mathbb{E}_{p(\bm{z}|\bm{x})} \hat{p} (y|k,\bm{z})].
\end{equation}
Recall that the joint distribution of FL is $p(k,y,\bm{x})=p(k)p(\bm{x},y|k)$. Thus, the global objective function $\mathcal{L}_{fl}$ is given by:
\begin{align}
        \mathcal{L}_{fl} 
        = \sum_{k \in \mathcal{K}} p(k)\mathbb{E}_{p(\bm{x},y|k)} [- \log \mathbb{E}_{p(\bm{z}|\bm{x})} \hat{p} (y|k,\bm{z})].
\end{align}

The goal of representation learning in FL is to obtain informative representation $Z$ that are useful for classifying their corresponding label $Y$.
Note that $Z$ and $Y$ are global random variables defined over all the clients.
However, the objective function $\mathcal{L}_{fl}$ for representation learning only aims to improve the local partial knowledge for local models as in the traditional FL framework.
This is because each client $k$ has access to only a local representation $\bm{z}$ that is sampled from $p(\bm{z}|\bm{x})$ given $\bm{x}$, which leads to insufficient knowledge in local models.

To improve the label knowledge of representations, an intuitive method is to maximize the mutual information between representation $Z$ and label $Y$, which is defined by $I(Y;Z) = \int_Y \int_Z p(y,\bm{z}) \text{log} \frac{p(y,\bm{z})}{p(y) p(\bm{z})} d\bm{z} dy$.
Nevertheless, computing the mutual information term $I(Y; Z)$ during local training is intractable since the absence of the global representation distribution $p(\bm{z} | y)$ leads to the intractability of $p(y,\bm{z})=p(y)p(\bm{z}|y)$. 
To overcome this issue, we derive an upper bound on the negative mutual information term in the following proposition by introducing an approximate variational distribution $q(\bm{z}|y)$  parameterized by a generative neural network $\bm{w}_g$, which will be detailed in Section \ref{generator}. 

\begin{proposition}
\label{pro_3}
    Let $q(\bm{z}|y)$ be the approximation of the global representation distribution $p (\bm{z}|y)$. The negative mutual information $-I(Y;Z)$ is upper bounded as follows:
    \begin{align}
    -I(Y;Z) \leq \mathcal{L}_{reg},
    \end{align}
    where 
    \begin{align}
    \label{l_reg}
        \mathcal{L}_{reg} = \sum_{k\in \mathcal{K}} p(k)\mathbb{E}_{p(y)} \mathbb{E}_{q(\bm{z}|y)} [- \log \hat{p}(y|k,\bm{z})].
    \end{align}
\begin{proof}
\text{See Appendix \ref{pro_3_proof}.}
\end{proof}
\end{proposition}
Note that the upper bound $\mathcal{L}_{reg}$ in Proposition \ref{pro_3} can be adopted as a regularizer in FL directly.
As such, we adopt $\mathcal{L}_{reg}$ as the proxy of the negative mutual information $-I(Y;Z)$ and focus on decreasing $\mathcal{L}_{reg}$ to minimize the negative mutual information $-I(Y;Z)$.
Specifically, in each communication round, each client downloads the generative neural network $\bm{w}_g$ from the server and uses it to output the global representations $\bm{z}$ given any class label $y$ and Gaussian noise $\bm{\epsilon}$. 
By including these global representations $\bm{z}$ in the local updating, clients can obtain more label information for representations.

\subsection{Client-Invariant Representation Learning for Federated Learning}
Under the feature-skewed scenarios in FL, training an encoder $\bm{w}_1$ that can extract client-invariant representations is non-trivial with distributed domains of data.
To address this problem, we introduce another regularization term to facilitate the extraction of client-invariant representation.

Intuitively, if the representation distribution of each class is the same among clients, i.e., $p(\bm{z}|k,y)=p (\bm{z} | k', y)$, $ \forall k, k' \in \mathcal{K}$, $k \neq k'$, the representation $\bm{z}$ is client-invariant.
We illustrate the sufficient conditions for the client-invariant representation extraction with the conditional mutual information term $I(\bm{z};k|y)$ in the following proposition.

\begin{proposition}
\label{remark_1}
    Given a label class $y$, a representation $\bm{z}$ is client-invariant if and only if $I(\bm{z};k|y)=0$, where $I(\bm{z};k|y)$ is the conditional mutual information between $\bm{z}$ and $k$.
\end{proposition}
\begin{proof}
    \text{See Appendix \ref{re_1_proof}.}
\end{proof}
According to Proposition \ref{remark_1}, the representation $Z$ is client-invariant when $I(Z;K|Y)=0$ given the label $Y$. 
Therefore, to achieve client-invariant representation learning, we need to reduce the conditional mutual information $I(Z;K|Y)= \sum_{k \in \mathcal{K}}\int_Y \int_Z p(k,y) p(\bm{z}|k,y) \log \frac{p(\bm{z}|k,y)}{p(\bm{z}|y)} d\bm{z} dy$ in FL.
However, this term is hard to compute in FL since it requires the conditional distribution $p(\bm{z}|k,y)$ and the global representation distribution $p(\bm{z}|y)$.
To tackle this issue, we derive an upper bound for the conditional mutual information $I(Z;K|Y)$ with the approximate global representation distribution $q(\bm{z}|y)$ in the following proposition.
\begin{proposition}
\label{proposition_1}
\label{client_inva}
    Let $q(\bm{z}|y)$ be the approximation of the global representation distribution $p (\bm{z}|y)$. The conditional mutual information $I(Z;K|Y)$ can be bounded as follows:
    \begin{equation}
    \begin{split}
        &I (Z; K|Y) \leq \mathcal{L}_{align},
    \end{split}
    \end{equation}
    where 
    \begin{align}
    \label{l_align}
        \mathcal{L}_{align} = \sum_{k \in \mathcal{K}} p(k)  \mathbb{E}_{p(\bm{x},y|k)} D_{\text{KL}} \big(p(\bm{z}|k,\bm{x}) \|  q(\bm{z}|y) \big),
    \end{align}
    and $D_{\text{KL}} \big(p(\bm{z}|k,\bm{x}) \|  q(\bm{z}|y) \big) = p(\bm{z}|k,\bm{x}) \log 
 \frac{p(\bm{z}|k,\bm{x})}{q(\bm{z}|y)}$ is the Kullback–Leible (KL) divergence between the local and approximate global representation distributions.
\end{proposition}
\begin{proof}
 \text{See Appendix \ref{pro_4_proof}.}
\end{proof}
Similar to $\mathcal{L}_{reg}$, the upper bound $\mathcal{L}_{align}$ in Proposition \ref{client_inva} can be used as a regularizer in FL for client-invariant feature extraction.
By using $\mathcal{L}_{align}$ as the proxy of $I(Z;K|Y)$, we can reduce $I(Z;K|Y)$ by minimizing $\mathcal{L}_{align}$, thereby achieving client-invariant feature extraction.
To be specific, each client downloads the approximate global representation distribution $q(\bm{z}|y)$ from the server, and aligns the local representation distribution $p(\bm{z}|k,\bm{x})$ with $q(\bm{z}|y)$ for the same class $y$ using the KL divergence in local training.
The representation distribution alignment helps to reduce the divergence between global and local representation distributions, thus encouraging clients to share the same representation distribution, i.e., $p(\bm{z}|k,\bm{x},y) \approx p(\bm{z}|k^\prime,\bm{x},y)$, $k,k^\prime \in \mathcal{K}$.


\begin{figure*}[pt]
    \centering \includegraphics[width=\textwidth]{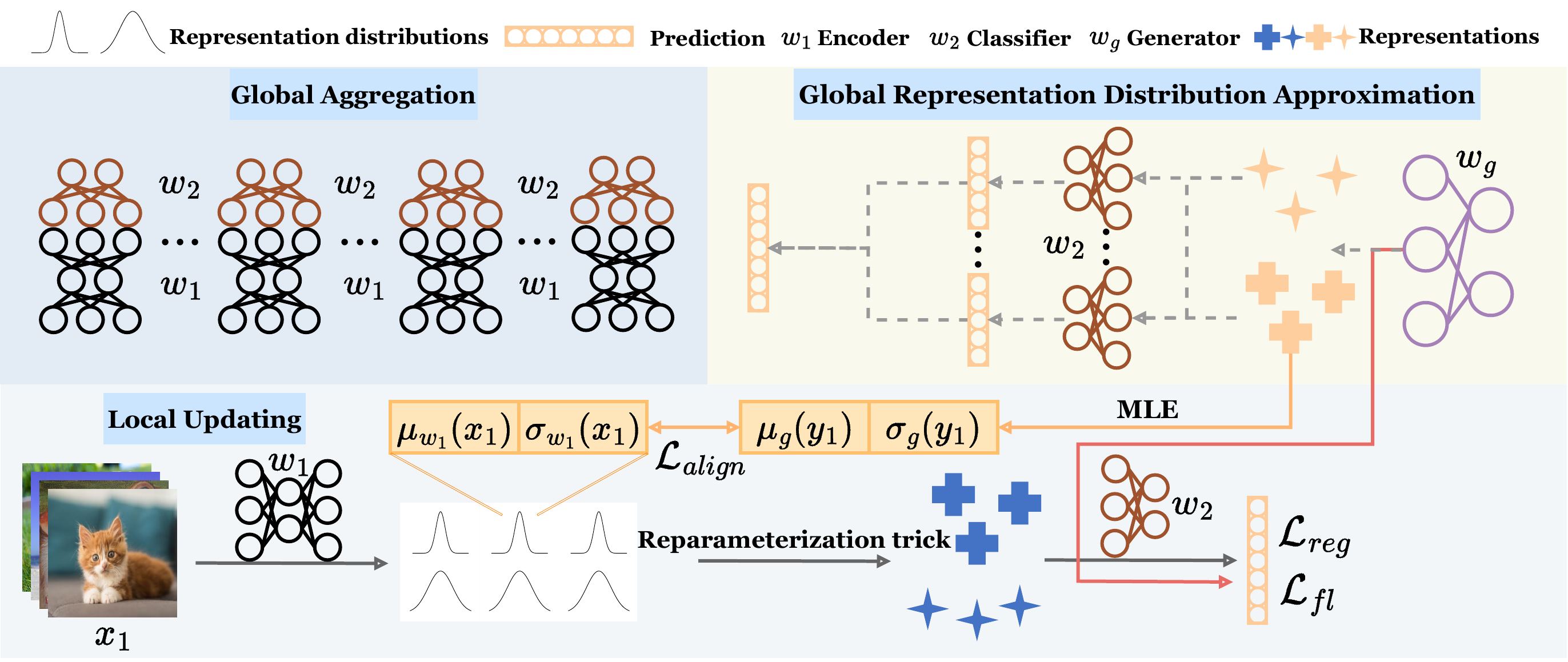}
    \caption{Illustration of the proposed method. Client side: Clients download the generator $\bm{w}_g$ and the approximate global representation distribution set $(\{\bm{\mu}_g(y),\bm{\sigma}_g(y)\}_{y=0}^{C-1})$ for local updating. Server side: the server uses the uploaded local classifiers to optimize the generator $\bm{w}_g$ for representation distribution approximation and conducts global aggregation.}
    \label{fig:framework}
\end{figure*}

\subsection{Global Representation Distribution Approximation}
\label{generator}
We note that the derived upper bounds of two mutual information terms, i.e., $\mathcal{L}_{reg}$ and $\mathcal{L}_{align}$, rely on the approximation of the global representation distribution $p(\bm{z} | y)$, i.e., $q(\bm{z} | y)$.
Nonetheless, the distributed local samples in clients make it difficult to approximate the global representation distribution $p(\bm{z} | y)$ in FL.
To tackle this issue, we propose to approximate the global representation distribution with a generative network for clients using a data-free mechanism.
In particular, we use a generator $G$ parameterized by $\bm{w}_g$ to generate the representation distribution $q_{\bm{w}_g}(\bm{z}|y)$ (we also omit $\bm{w}_g$ for notation simplicity), which is then optimized at the server in each communication round.
Since the global model obtained by aggregating local models loses the local knowledge, we exploit the ensemble predictions of the local models to optimize the generator $G$ by minimizing the following objective function:
\begin{align}
\label{g_train}
    \mathcal{L}_g = \mathbb{E}_{ p(y)} \mathbb{E}_{ q(\bm{z}|y)}  \left[- \log \sum_{k\in\mathcal{K}} p(k) \hat{p} ({y} |k, \bm{z}) \right],
\end{align}
where $q(\bm{z}|y) \equiv G(y,\bm{\epsilon} | \bm{\epsilon} \sim \mathcal{N}(0,\bm{I}))$ is the generated representation distribution, and $\hat{p} ({y} |k, \bm{z})$ is the predicted logit output by the local model from client $k$.
FedGen \cite{fedgen} requires additional training data distribution for representation distribution approximation, which is vulnerable to gradient inversion attacks \cite{gradient_inversion} and brings extensive privacy risks.
Instead, our proposed data-free mechanism only leverages the local models to approximate the global representation distribution and does not bring extra privacy exposure. 
For any given class label $y$, by maximizing the agreement of local models, the representation $\bm{z}$ generated by the generator network achieves the maximum consensus of them, thereby effectively approximating the global distribution representation.

\subsection{Training Details and Overall Framework of FedCiR}
\label{framework}

By integrating $\mathcal{L}_{fl}$ and the proposed two regularization terms $\mathcal{L}_{req}$ and $\mathcal{L}_{align}$ obtained by the Proposition \ref{pro_3} and Proposition \ref{client_inva}, respectively, the global objective function of the proposed method FedCiR is given by: 
\begin{align}
    \mathcal{L}_{FedCiR} = \mathcal{L}_{fl} + \lambda_{reg} \mathcal{L}_{reg} + \lambda_{align} \mathcal{L}_{align},
\end{align}
where $\lambda_{reg}$ and $\lambda_{reg}$ are the hyperparameters for the regularization terms.
With these two regularizers, FedCiR can achieve label-discriminative and client-invariant feature extraction and solve the feature shift issue in FL.
In the following, we elaborate on the training details of FedCiR.

\textbf{KL divergence between local and approximate global representation distributions} $D_{\text{KL}} \big(p(\bm{z}|k,\bm{x}) \|  q(\bm{z}|y) \big)$:
To compute $D_{\text{KL}} \big(p(\bm{z}|k,\bm{x}) \|  q(\bm{z}|y) \big)$ in \eqref{l_align}, following \cite{para_trick}, we assume that the global representation distribution $q(\bm{z}|y)$ is a Gaussian distribution, which can be parameterized by the mean $\bm{\mu}_g (y)$ and standard deviation $\bm{\sigma}_g (y)$ for each class $y$.
Note that the generator $G$ outputs the representations samples instead of the distribution $q(\bm{z}|y)$.
Hence, we estimate the mean $\bm{\mu}_g (y)$ and standard deviation $\bm{\sigma}_g (y)$ for each class $y$ at the server using the maximum likelihood estimation (MLE) method.
As such, the KL divergence \cite{kl_divergence} between $p(\bm{z}|k,\bm{x})$ and $q(\bm{z}|y)$ in \eqref{l_align} can be computed as follows:
\begin{align}
         &D_{\text{KL}} (p(\bm{z}|k, \bm{x}) \| q(\bm{z}|y))
         = \log \bm{\sigma}_g (y) - \log  \bm{\sigma}_{\bm{w}_1} (y) \nonumber \\ 
        & \quad \quad \quad \quad \quad+ \frac{\bm{\sigma}_{\bm{w}_1}^2 (\bm{x}) + (\bm{\mu}_{\bm{w}_1} (\bm{x})-\bm{\mu}_g (y))^2}{2 \bm{\sigma}_g^2 (y)} - \frac{1}{2}.
        \label{kl}
\end{align}


\begin{algorithm}
    \renewcommand{\algorithmicrequire}{\textbf{Input:}}
    \renewcommand{\algorithmicensure}{\textbf{Output:}}
    \caption{FedCiR}
    \label{algo}
    \begin{algorithmic}[1]
    \REQUIRE communication round $\mathcal{T}$; local steps $E$; training steps for generator $E_g$; class number $C$; client set $\mathcal{K}$; local distributions $\{\mathcal{D}_k\}_{k \in \mathcal{K}}$; global model $\bm{w}$; generator $\bm{w}_g$; learning rate $\gamma$, $\gamma_g$; .
        \STATE \textbf{Initialization}: model $\bm{w} = \bm{w}_1 \circ \bm{w}_2$ and generator $\bm{w}_g$
        \FOR{ each communication round $t=1,\dots, \mathcal{T}$}
        \STATE \textbf{Server Executes:}
        \STATE Sample active client set $\mathcal{K}_t$ from $\mathcal{K}$
        \FOR{$k \in \mathcal{K}_t$ \textbf{in parallel}}
        \STATE $\bm{w}_k^t \leftarrow$ \text{ClientUpdate} ($\bm{w}_t, \bm{w}_g, \{\bm{\mu}_g(y),\bm{\sigma}_g(y)\}_{y=0}^{C-1}$)
        \ENDFOR
        \FOR{each step $e=1, \dots, E_g$}
        \STATE $\bm{w}_{g,e}=\bm{w}_{g,e-1} - \gamma_g \mathcal{L}_g (\bm{w}_{g,e-1}) $ \quad $\triangleright$ Update generator $\bm{w}_g$ by (\ref{g_train})
        \ENDFOR
        \STATE Compute the global representation distribution sets $\{\bm{\mu}_g(y),\bm{\sigma}_g(y)\}_{y=0}^{C-1}$ using MLE
        \STATE $\bm{w}_{t+1} = |\frac{1}{\mathcal{K}_t}| \sum_{k \in \mathcal{S}_t}  \bm{w}_k^t$ \quad $\triangleright$ Global aggregation
        \ENDFOR
        \end{algorithmic}
        \textbf{def} \text{ClientUpdate($\bm{w}_t$, $\bm{w}_g$, $\{ \bm{\mu}_g(y),\bm{\sigma}_g(y)\}_{y=0}^{C-1}$):} 
        
        \begin{algorithmic}[1]
        \FOR{ each step $e=1,\dots,E$}
        \STATE $\bm{w}_{k}^{t,e}=\bm{w}_{k}^{t,e-1} - \gamma \nabla \mathcal{L}_k(\bm{w}_{k}^{t,e-1})$ \quad $\triangleright$ Update local models by (\ref{local_update}) 
        \ENDFOR
        \RETURN $\bm{w}_{k}^{t,E}$
        \end{algorithmic}
\end{algorithm}

The training process of the FedCiR framework, which is illustrated in Figure \ref{fig:framework} and summarized in Algorithm \ref{algo}, involves iterative steps where clients carry out local training, and the server performs global aggregation and global representation distribution approximation. The details of these two steps are provided as follows.

\textbf{Local training:}
In each communication round $t= 1, \cdots, \mathcal{T}$, a subset of clients $\mathcal{K}_t$ from $\mathcal{K}$ is active, and they download the global model and global representation distribution set $\{\bm{\mu}_g(y),\bm{\sigma}_g(y)\}_{y=0}^{C-1}$ from the server.
At each step $e \in [E]$ of the $t$-th local training, the active client $k$ samples a mini-batch of data $\xi_k^{t,e} = \{ \bm{x}_i, y_i \}_{i=1}^{B}$ with $B$ samples and input them into the encoder $\bm{w}_1$ to obtain the representation distribution $p(\bm{z}|\bm{x}_i)$ parameterized by $\bm{\mu}_{\bm{w}_1}(\bm{x}_i)$ and $\bm{\sigma}_{\bm{w}_1}(\bm{x}_i)$ for each data sample $\bm{x}_i$.
These representation distributions are used to compute the KL divergence term in \eqref{kl}.
Besides, we sample a representation $\bm{z}$ from the distribution $p(\bm{z}|k,\bm{x}_i)$, which is then input into the classifier to compute the loss term $\mathcal{L}_{fl}$.
To compute $\mathcal{L}_{reg}$ in \eqref{l_reg}, we sample a mini-batch size $B$ of Gaussian noise vectors $ \{\bm{\epsilon}_i | \bm{\epsilon}_i \sim \mathcal{N}(0,\bm{I})\}_{i=1}^B$ and label instances $\{ y_i\}_{i=1}^B$ over a uniform distribution.
The samples are input into the generator for generating the global representations, which are then input into the classifier for regularizing the local classifier.
Following the classical FedAvg method described in Section \ref{Preliminary}, the local objective function with the mini-batch of samples $\xi_k^{t,e}$ for client $k$ is given by:
\begin{align}
\label{local_update}
    \mathcal{L}_{k} (\bm{w}) =& \mathbb{E}_{(\bm{x}_i, y_i) \in \xi_k^{t,e} } [- \log \mathbb{E}_{p(\bm{z}|{\bm{x}_i})} \hat{p} (y|k,\bm{z}) \nonumber \\ 
    &+ \lambda_{reg} \mathbb{E}_{p(y)} \mathbb{E}_{q(\bm{z}|y)} [- \log \hat{p}(y|k,\bm{z})] \nonumber \\
    &+ \lambda_{align}  D_{\text{KL}} (p(\bm{z}|k,\bm{x}_i) \| q(\bm{z}|y_i))].
\end{align}

\textbf{Gloabl aggregation and global representation distribution approximation:}
After local training, the clients share the local models with the server as in FedAvg.
Besides the global model aggregation, the server optimizes the generator $\bm{w}_g$ with the help of local classifiers for $E_g$ steps.
Specifically, at each step $e\in [E_g]$ of the generator updating, the server samples a mini-batch size $B$ of Gaussian noise vectors $ \{\bm{\epsilon}_i | \bm{\epsilon}_i \sim \mathcal{N}(0,\bm{I})\}_{i=1}^B$ and label instances $\{y_i\}_{i=1}^B$ over a uniform distribution.
The server obtains the representations by inputting the Gaussian noise vectors and label instances into the generator $\bm{w}_g$.
To encourage the generator to maximize the agreement of all the clients, the server inputs the representations into all the uploaded local classifiers to compute the ensemble predictions.
Then, the generator is optimized by minimizing the cross-entropy loss between the ensemble predictions and the corresponding labels.
To relieve the computation burden for clients, the server updates the global representation distribution set $\{\bm{\mu}_g(y),\bm{\sigma}_g(y)\}_{y=0}^{C-1}$ for each class via MLE, which will be downloaded by the clients along with the global model in the next round.

\begin{table*}
\centering
\begin{tabular}{c||ccccc}
\hline
\rowcolor[gray]{0.8} \multicolumn{6}{c}{\textbf{Top-1 Test Accuracy}} \\
\hline \hline
\textbf{Accuracy (\%)}
&\multicolumn{1}{c}{\textbf{DomainNet}}&\multicolumn{1}{c}{\textbf{PACS}}&\multicolumn{1}{c}{\textbf{Office-Caltech-10}} &\multicolumn{1}{c}{\textbf{Camelyon17}} &\multicolumn{1}{c}{\textbf{Iwildcam}}\\ \hline \hline
{Centralized training} & $78.07 \pm 0.65$ & $81.34 \pm 0.87$ &
$81.49 \pm 1.10$ & $94.94 \pm 0.12$ & $87.59 \pm 0.51$ \\
FedAvg \cite{fl} & $60.03 \pm 1.43$ & $52.61 \pm 1.82$ & $62.73 \pm 0.32$ & $79.88 \pm 0.67$ & $85.26 \pm 0.62$  \\
FedProx \cite{fedprox} & $59.96 \pm 0.82$ & $55.14 \pm 1.42$ & $64.37 \pm 1.46$ & $79.95 \pm 0.82$  & $84.75 \pm 0.83$ \\
FedProc \cite{fedproc} & $61.61 \pm 0.87$ & $55.16 \pm 1.57$ & $63.58 \pm 1.78$ & $80.02 \pm 0.80$ &  $85.19 \pm 0.58$ \\
FedGen \cite{fedgen} & $63.74 \pm 1.12$ & $56.01 \pm 1.41$ & $62.80 \pm 1.84$ & $79.52 \pm 0.81$ & $84.81 \pm 0.89$ \\
FedSR \cite{fedsr} & $63.02 \pm 1.52$ & $55.18 \pm 2.61$ & $62.95 \pm 2.13$ & $80.21 \pm 0.62$ & $85.27 \pm 0.92$ \\
FedReg (Ours) & \underline{$ {64.78} \pm 1.34$} & $57.59 \pm 1.80$ & $63.45 \pm 1.32$ & $80.33 \pm 0.67$ & \underline{$86.14 \pm 0.52$} \\
FedAlign (Ours) & $63.95 \pm 0.81$ & \underline{$ 57.83 \pm 1.52$} & \underline{ $65.81 \pm 1.42$} & \underline{$80.38 \pm 0.74$} & $85.60 \pm 0.70$ \\
FedCiR (Ours) & $\bm{65.18} \pm \bm{1.12}$ & $\bm{58.81} \pm \bm{1.48}$ & $\bm{66.14} \pm \bm{1.45}$ & $ \bm{80.45} \pm \bm{0.71}$ & $\bm{86.23} \pm \bm{0.47}$ \\ \hline
\end{tabular}
\caption{Test accuracy (\%) of different algorithms on various datasets. The results in \textbf{bold}
indicate the best performance, and the second best results are \underline{underlined}.}
\label{main_results}
\end{table*}

\section{Experiments}
\label{experiments}
In this section, we evaluate the effectiveness of the proposed FedCiR framework in solving the feature shift problem in FL.
We first summarize the experiment setup in Section \ref{setup}, and compare FedCiR with several state-of-the-art (SOTA) FL algorithms on benchmark and real-world datasets in Section \ref{performance}. 
In Section \ref{ablation study}, ablation studies are conducted to demonstrate the sensitivity of FedCiR on hyperparameters.
To illustrate the effectiveness of FedCiR in client-invariant representation training as discussed in Theorem \ref{theorem_2}, we further compare FedCiR with FedAvg in terms of local risks and $\mathcal{H}$-divergence in Section \ref{insight}.
An overhead analysis is conducted in Section \ref{overhead} to demonstrate that FedCiR is energy-friendly in terms of computation, communication, and GPU memory.

\subsection{Experiment Setup}
\label{setup}
\textbf{Baselines:}
Apart from FedAvg \cite{fl}, we compare FedCiR against several SOTA FL algorithms in solving the non-IID issue, including FedProx \cite{fedprox}, FedProc \cite{fedproc}, FedGen \cite{fedgen}, and FedSR \cite{fedsr}.
To gain more insights from FedCiR, we develop two more
baselines: FedReg and FedAlign. Specifically, FedReg only improves the global knowledge for local models with the regularizer $\mathcal{L}_{reg}$, while FedAlign only utilizes the regularizer $\mathcal{L}_{align}$ to force the representations to be client-invariant.

\begin{figure}
    \centering
    \includegraphics[width=0.4\textwidth]{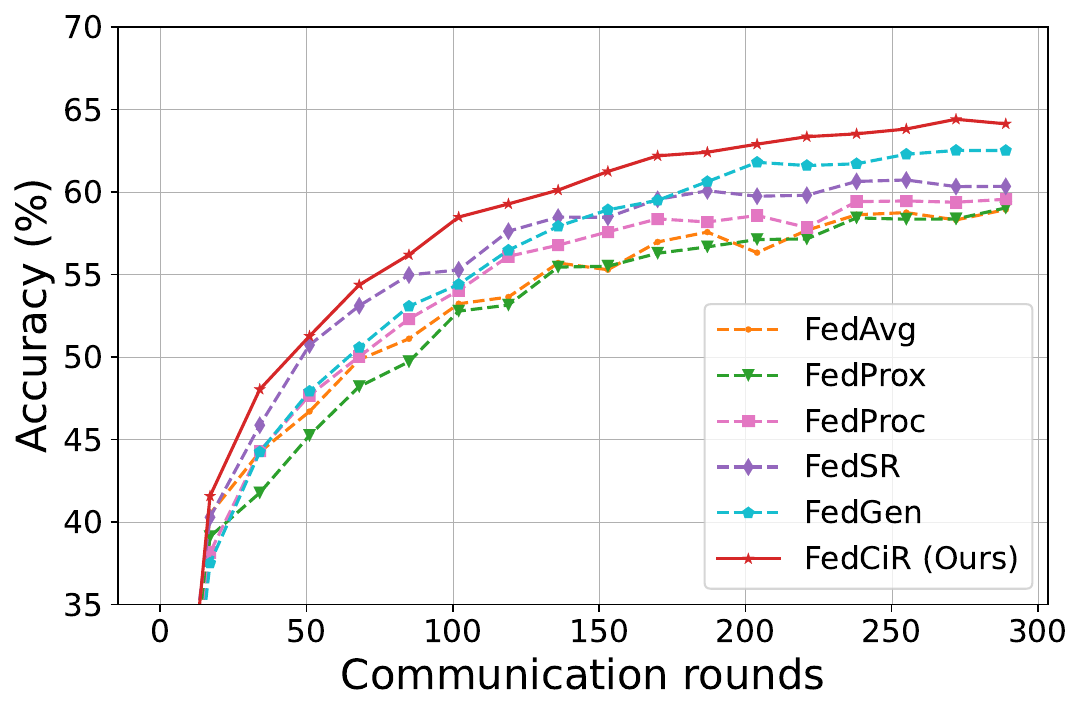}
    \caption{Convergence performance of different algorithms on the DomainNet dataset.}
    \label{Fig.conver_domainnet}
\end{figure}

\begin{figure}
    \centering
    \includegraphics[width=0.4\textwidth]{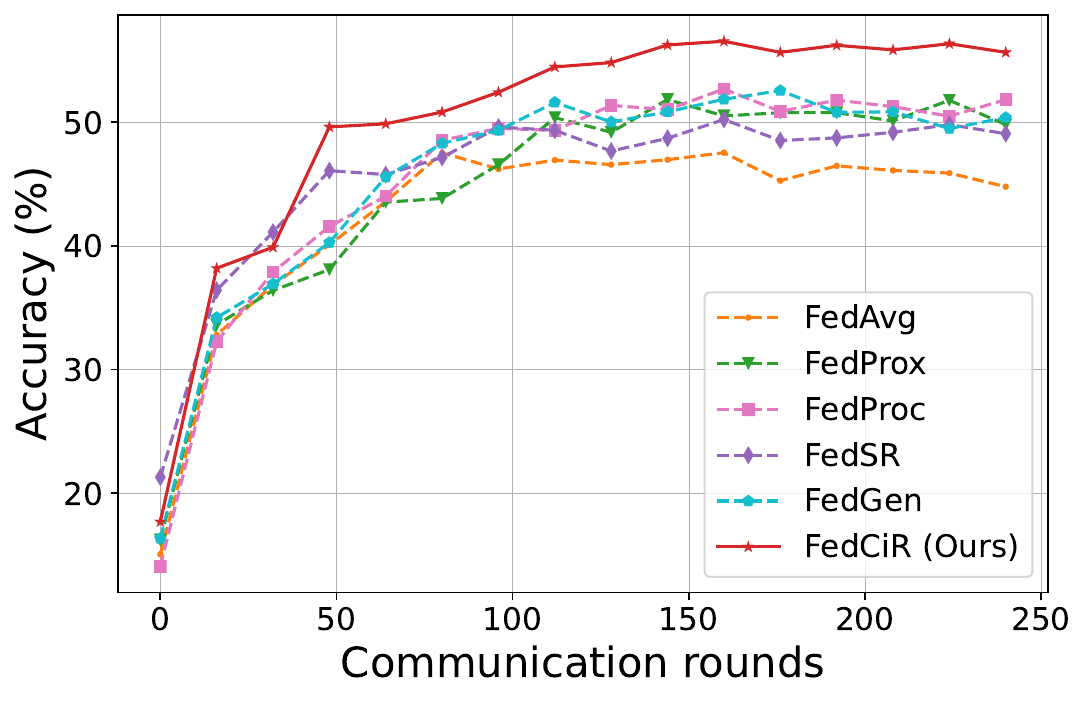}
    \caption{Convergence performance of different algorithms on the PACS dataset.}
    \label{Fig.conver_pacs}
\end{figure}

\textbf{Datasets:}
We consider both benchmark and real-world datasets, including DomainNet \cite{domainnet}, PACS \cite{PACS}, Office-Caltech-10 \cite{office_cal}, Camelyon17 \cite{camelyon17}, and Iwildcam \cite{iwildcam}, which contain multiple domains of data, e.g., photos, art paintings, cartoon images, and sketch images on PACS.
For FL, We consider the feature shift scenarios where each client has only one domain of data that are resized to $128 \times 128$ pixels.
The Camelyon17 dataset is a binary classification task for detecting tumors in varying hospitals and the Iwildcam dataset is a multi-class classification task for wildlife animal images captured by cameras at different locations. 
For the Iwildcam dataset, the data among cameras (clients) are both label-skewed and feature-skewed, and most of them are invalid background data and
useless for training.
Therefore, we only use ten classes of samples in ten of the cameras that have the largest data volume.

Moreover, apart from the feature-skewed data distribution, we also evaluate the effectiveness of FedCiR on label-skewed data distribution, where the data in each domain are distributed to multiple clients (number of clients / number of domains) according to the Latent Dirichlet Sampling with varying Dirichlet parameters $\beta$.

\textbf{Model structures:}
We use ResNet18 \cite{resnet18} as the backbone for all the datasets. 
For FedCiR and FedGen, the generator is composed of two fully connected layers with ReLU and BatchNorm layers between them.

\textbf{Hyperparameters:}
In the main experiments, we set the local updating steps $E=20$ and communication rounds $T=300$.
During local training, we use SGD with batch size 32, learning rate $0.01$, momentum 0.9, and weight decay $5\times 10^{-4}$, as the local optimizer.
For FedCiR and FedGen, the server optimizes the generator for 5 steps $(E_g=5)$ in each communication round.
More details of hyperparameters for FedCiR and baselines are available in Appendix \ref{hyper}.

\begin{figure}
    \centering
    \includegraphics[width=0.4\textwidth]{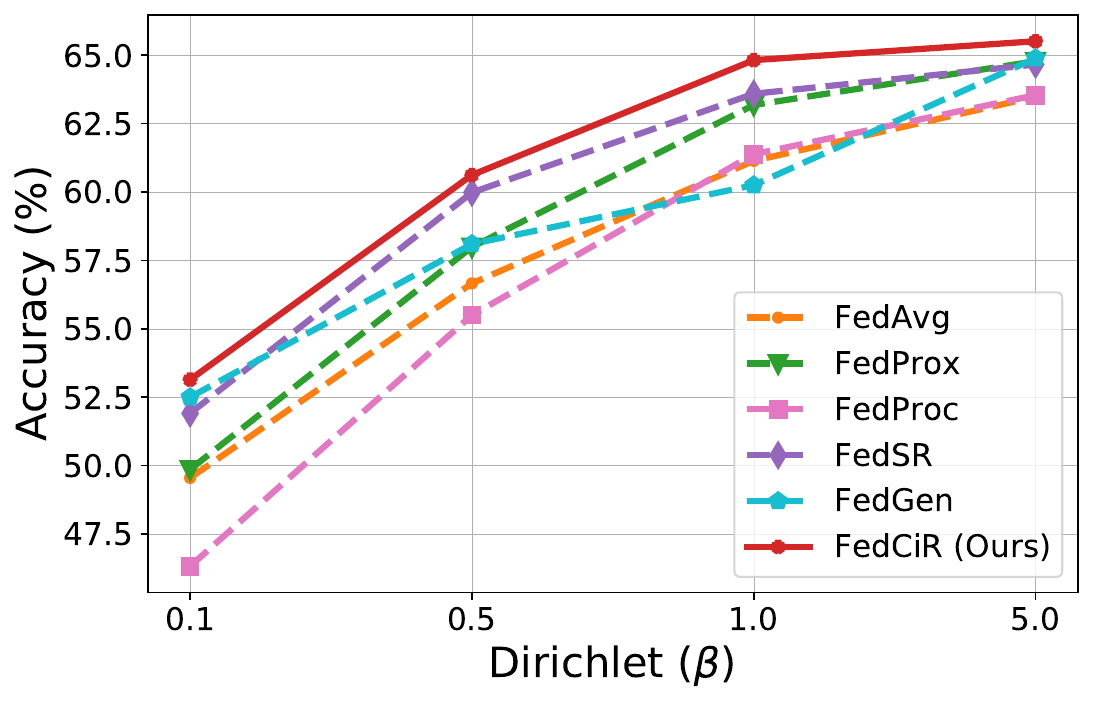}
    \caption{Test accuracy ($\%$) of different algorithms under various non-IID settings with 50 clients on DomainNet.}
    \label{Fig.hete_domainnet}
\end{figure}

\begin{figure}
    \centering
    \includegraphics[width=0.4\textwidth]{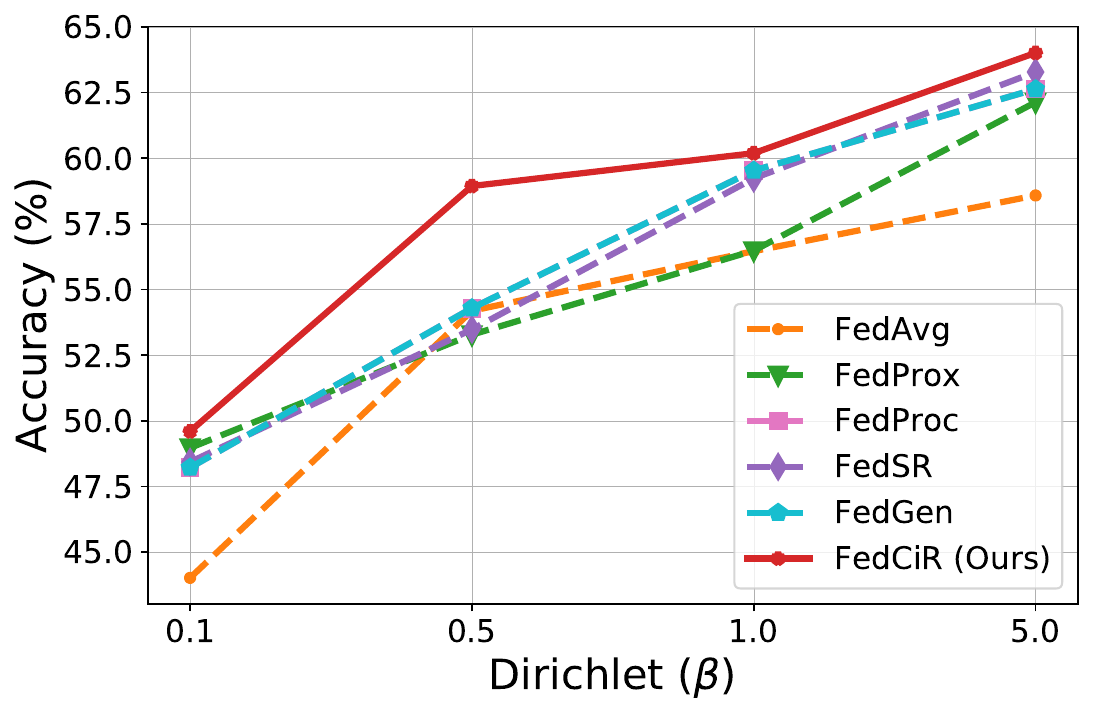}
    \caption{Test accuracy ($\%$) of different algorithms under various non-IID settings with 50 clients on PACS.}
    \label{Fig.hete_pacs}
\end{figure}

\subsection{Performance Evaluation}
\label{performance}
\textbf{Test accuracy:}
Table \ref{main_results} presents the test accuracy of FedCiR and the baselines, where the first row shows the performance of centralized training.
As shown in Table \ref{main_results}, FedCiR outperforms the baselines on all the datasets, surpassing FedGen by $1.24\%$ and $1.8\%$ on DomainNet and PACS, respectively.
FedGen employs a data-free knowledge distillation mechanism to regularize the local models, resulting in a better performance compared with FedAvg and FedProx.
FedProc and FedSR rely on the feature alignment to minimize the local model divergence. However, the experiment results reveal that they exhibit instability with different random seeds on some datasets, e.g., DomainNet and PACS.
This is because the average prototypes and representation networks in the server cannot accurately identify the precise features for alignment in the feature space as compared with our proposed data-free mechanism for representation distribution approximation.
In addition, FedCiR is consistently superior to the baselines on the real-world datasets: Camelyon17 and Iwildcam, which demonstrates its effectiveness in practical applications.

\textbf{Convergence rate:}
Fig. \ref{Fig.conver_domainnet} and Fig. \ref{Fig.conver_pacs} illustrate the convergence rate of FedCiR and baselines on DomainNet and PACS, respectively.
FedCiR achieves a significantly faster convergence rate compared with the baselines.
While FedGen may enjoy a higher learning efficiency in some settings by relying on effective proxy data, our proposed approach can explicitly align the representation distributions of all the clients and extract client-invariant representations, thereby achieving the fastest convergence rate.


\begin{figure}
    \centering
    \includegraphics[width=0.4\textwidth]{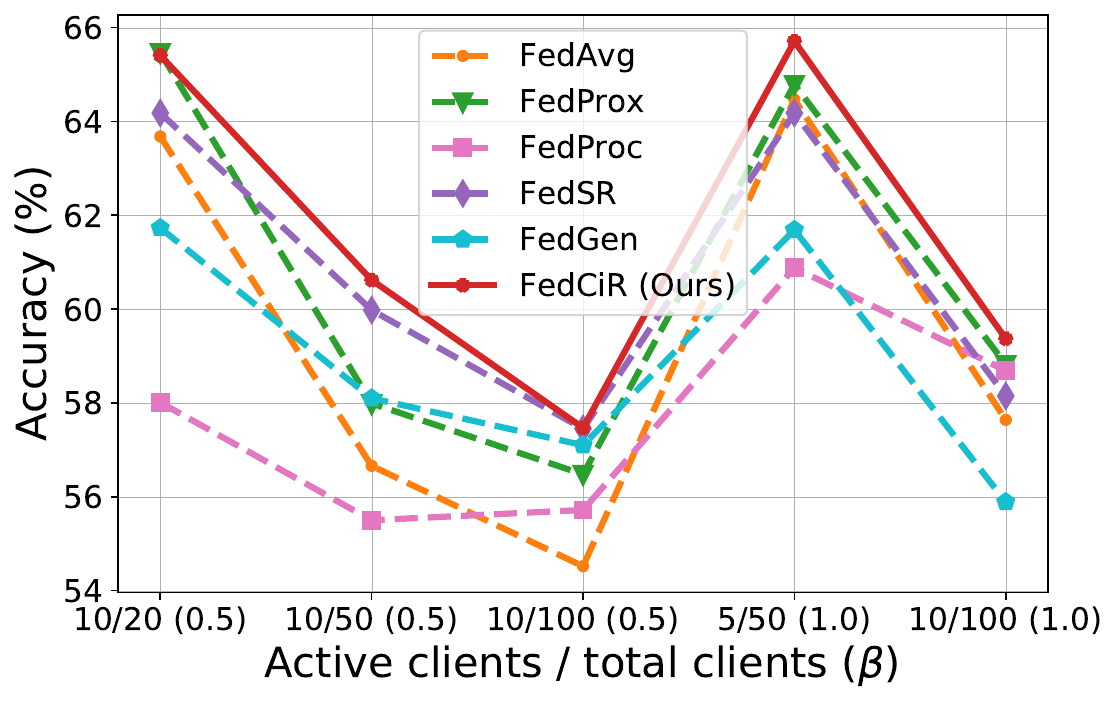}
    \caption{Test accuracy ($\%$) of different methods with varying client number and active client ratio on DomainNet. The values of Dirichlet parameter $\beta$ are given in parentheses.}
    \label{Fig.client_sel_domainnet}
\end{figure}

\begin{figure}
    \centering
    \includegraphics[width=0.4\textwidth]{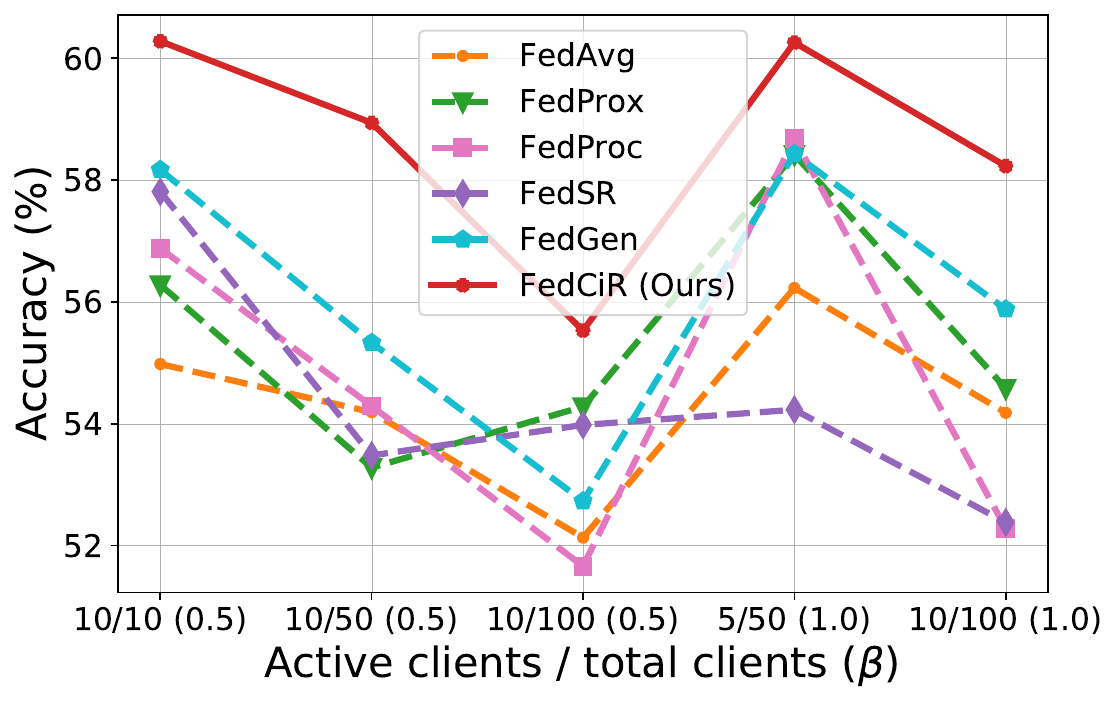}
    \caption{Test accuracy ($\%$) of different methods with varying client number and active client ratio on PACS. The values of Dirichlet parameter $\beta$ are given in parentheses.}
    \label{Fig.client_sel__pacs}
\end{figure}

\begin{table}[h]
    \centering
    \resizebox{0.48\textwidth}{!}{
    \begin{tabular}{c||ccccc}
    \hline
    \rowcolor[gray]{0.8}\multicolumn{6}{c}{\textbf{Robustness of Hyperparameters (E=100) }}\\
    \hline \hline
       $ \lambda_{reg}$ ($\lambda_{align}$=5e-6) & 0.0 & 0.1 & 0.3 & 0.5 & 1.0  \\
       Accuracy ($\%$)  & $63.53$ & $67.73$ & $65.98$ & $67.58$ & $67.53$ \\
    \hline \hline
    $ \lambda_{align}$ ($\lambda_{reg}$=0.5) & 0.0 & 1e-6 & 5e-6 & 1e-5 & 5e-5 \\
    Accuracy ($\%$) & $63.83$ & $64.43$ & $67.58$& $67.47$ & $67.39$ \\
    \hline
    \end{tabular}}
        \caption{Test accuracy ($\%$) of FedCiR using varying hyperparameters.}
    \label{tab:hyper}
\end{table}

\textbf{Data heterogeneity:}
In addition to the feature-skewed scenarios, we further evaluate the performance of all the methods in various label-skewed settings on DomainNet and PACS datasets. 
As shown in Fig. \ref{Fig.hete_domainnet} and Fig. \ref{Fig.hete_pacs}, FedCiR outperforms the baselines in all the scenarios, which demonstrates that our proposed method can handle all kinds of non-IID scenarios, including those with feature-skewed and label-skewed data distributions.
As the non-IID degree decreases, i.e., as $\beta$ increases, the test accuracy of all the methods improves as expected.
Moreover, FedCiR obtains competitive results under highly-skewed scenarios, e.g., it achieves nearly $5\%$ higher accuracy than the second-best method on PACS when $\beta=0.5$.


\subsection{Sensitivity Analysis}
\label{ablation study}
\textbf{Impact of client number and active ratio:}
We investigate the impact of client number and active ratio for FedCiR and the baselines.
Fig \ref{Fig.client_sel_domainnet} and Fig. \ref{Fig.client_sel__pacs} illustrate that, under different total client and active client numbers, FedCiR achieves the best performance in all the settings.
For example, FedCiR outperforms the second-best algorithm by 3.5$\%$ on PACS when the total client number is 50 and the active client number is 10.
Notably, the accuracy of FedProc and FedSR is sensitive to the numbers of clients due to the unstable prototypes and representation networks averaged over them.
In contrast, FedCiR maintains stable and competitive performance across all the settings, highlighting the effectiveness and stability of client-invariant representation learning.

\begin{figure}
    \centering
    \includegraphics[width=0.4\textwidth]{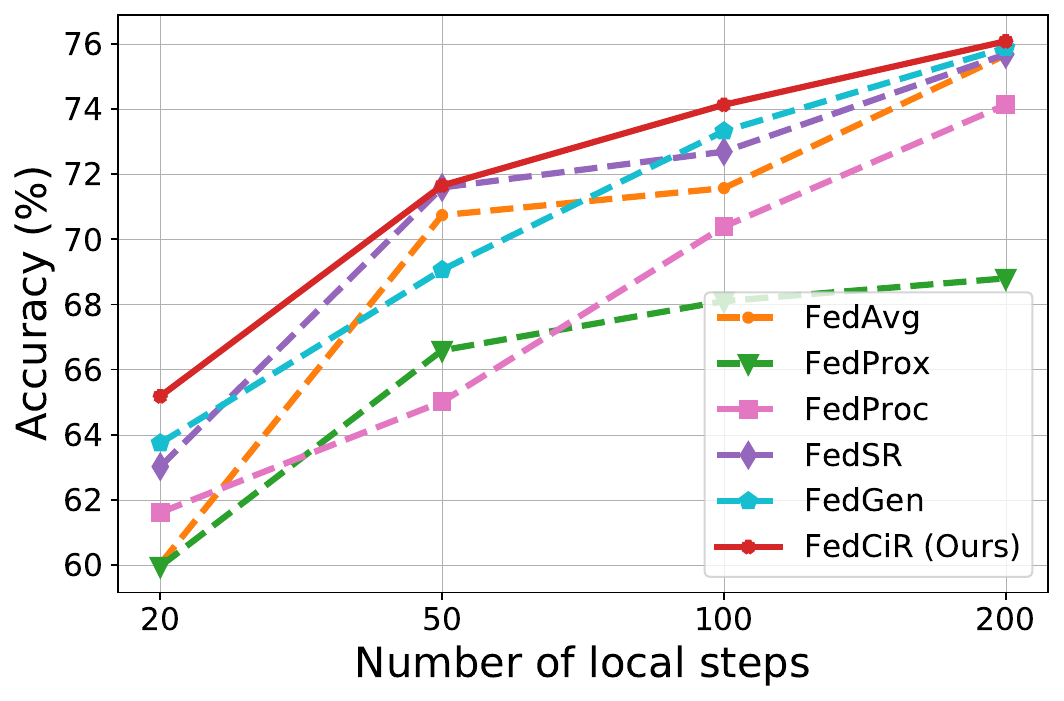}
    \caption{Test accuracy ($\%$) of different methods with varying local updating steps on DomainNet.}
    \label{Fig.local_epoch_domainnet}
\end{figure}

\begin{figure}
    \centering
    \includegraphics[width=0.4\textwidth]{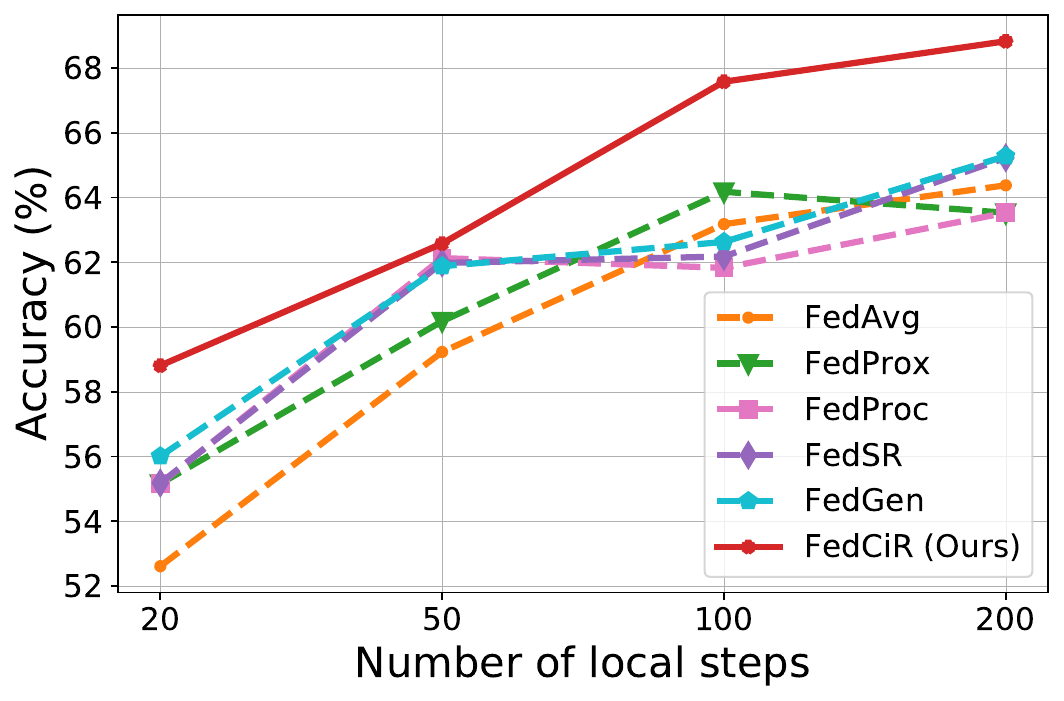}
    \caption{Test accuracy ($\%$) of different methods with varying local updating steps on PACS.}
    \label{Fig.local_epoch_pacs}
\end{figure}

\textbf{Impact of local steps:}
We study the effect of the number of local updating steps on the accuracy of all the methods on DomainNet and PACS datasets.
As depicted in Fig. \ref{Fig.local_epoch_domainnet} and Fig. \ref{Fig.local_epoch_pacs}, appropriate large values of local updating steps promote the accuracy performance for all the methods.
Furthermore, FedCiR is robust against all the baselines at different levels of local updating steps, particularly on the PACS dataset.
Notably, when the number of local steps exceeds 100, FedCiR outperforms the second-best baseline by more than 3$\%$, indicating its high training reliability and efficiency. 


\textbf{Robustness of hyperparameters:}
To evaluate the influence of hyperparameter selection, we conduct experiments with different hyperparameters in FedCiR on PACS.
Results in Table \ref{tab:hyper} show that FedCiR is insensitive to the hyperparameters $\lambda_{reg}$, achieving an accuracy of about $68\%$ with $\lambda_{reg} \in [0.1,1.0]$ and $\lambda_{align}=5 \times 10^{-6}$.
Without the proposed regularization term $\mathcal{L}_{reg}$, i.e., $\lambda_{reg}=0.0$, FedCiR only achieves 63.53$\%$ accuracy, highlighting the effectiveness of our proposed term for knowledge augmentation.
We also find that selecting an appropriate large value for $\lambda_{align}$, e.g., $\lambda_{align}> 1\times 10^{-6}$, yields better performance for FedCiR, whereas using an extremely small value of $\lambda_{align}$ cannot extract sufficient client-invariant features for classification, resulting an accuracy of only $63.83\%$ when $\lambda_{align}=0$.


\begin{figure}
    \centering
    \includegraphics[width=0.4\textwidth]{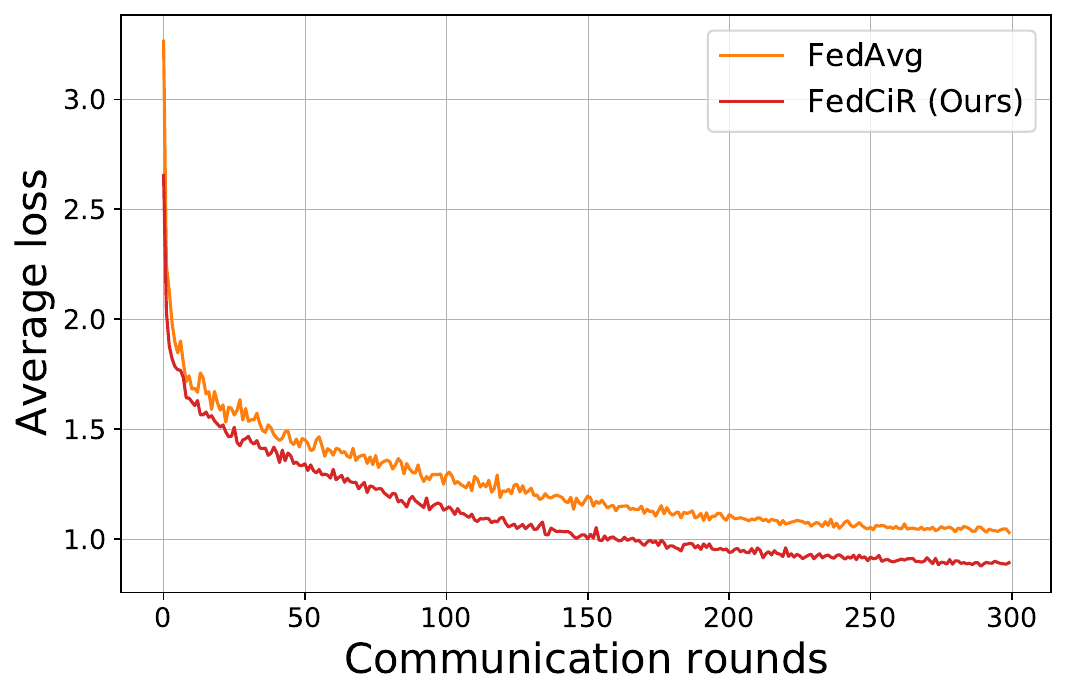}
    \caption{Average local risks on DomainNet.}
    \label{Fig.individual_error}
\end{figure}

\begin{figure}
    \centering
    \includegraphics[width=0.4\textwidth]{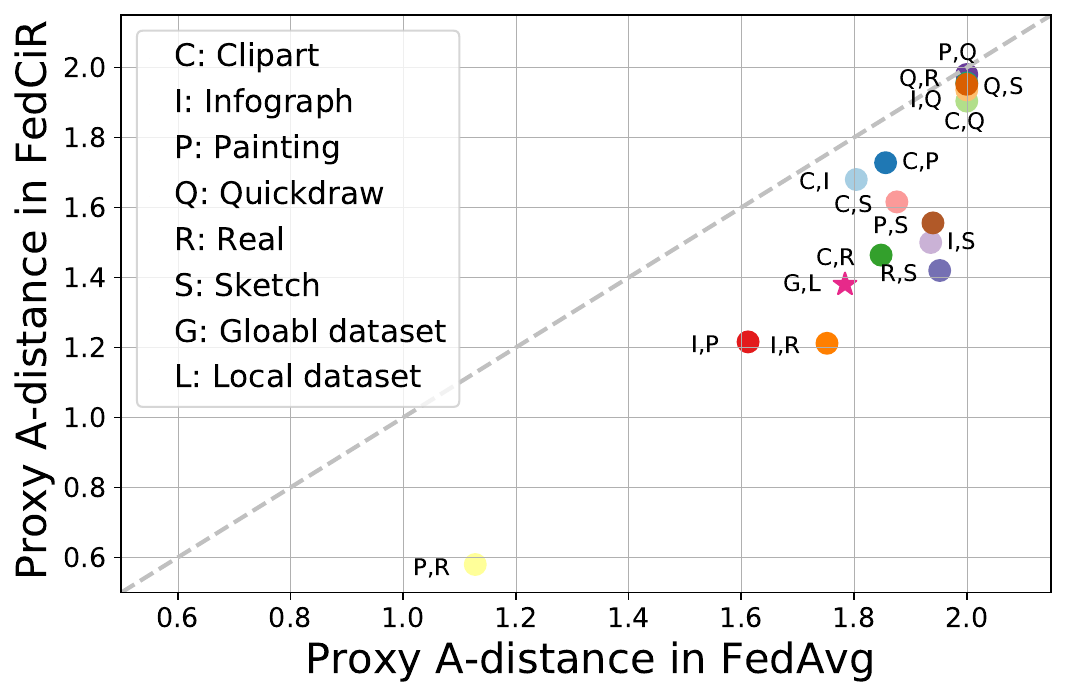}
    \caption{Proxy A-distance (PAD) between any pairs of datasets on DomainNet. \emph{Pink star:} Average PAD between the global mixture dataset and all the local datasets. \emph{Circles:} PAD between any pairs of local datasets. The positions of the dots below the dashed diagonal indicate that FedCiR achieves a smaller PAD in FedCiR compared with FedAvg.}
    \label{Fig.h_div}
\end{figure}

\subsection{Insight Analysis}
\label{insight}
To verify the effectiveness of FedCiR in client-invariant representation learning as discussed in Theorem \ref{theorem_2}, we empirically calculate the reduction of local risks and $\mathcal{H}$-divergence between local and global datasets compared with FedAvg on DomainNet.
Similar to \cite{domain_adaptation_1,theory_da}, we approximate the $\mathcal{H}$-divergence by discriminating a pair of datasets using a classifier.
Specifically, we label two datasets $\mathcal{D}_i$ and $\mathcal{D}_j$ as 0 and 1, respectively, and train a classifier, e.g., a linear SVM, to distinguish between them and calculate the test error $\epsilon$.
The proxy A-distance (PAD) between $\mathcal{D}_i$ and $\mathcal{D}_j$, defined as $2(1-2 \epsilon)$, is used to approximate the $\mathcal{H}$-divergence.

Fig. \ref{Fig.individual_error} presents the comparison of the average local risks of clients on DomainNet for FedCiR and FedAvg. 
The results show that FedCiR consistently outperforms FedAvg in terms of lower risks in clients.
This is because the client-invariant representations induced by FedCiR reduce the discrepancy among local models.
Consequently, the aggregated global model becomes more generalized for local datasets, resulting in lower local errors.
Furthermore, as shown in Fig. \ref{Fig.h_div}, the average PAD between global and local datasets is smaller in FedCiR than in FedAvg (i.e., the dot is below the diagonal), which verifies the effectiveness of FedCiR in reducing the $\mathcal{H}$-divergence between global and local datasets. 
Moreover, the PAD between any pairs of local datasets is also lower in FedCiR than in FedAvg, explicitly demonstrating that FedCiR can extract more client-invariant representations compared with FedAvg.

\subsection{Overhead Analysis}
\label{overhead}
Finally, we evaluate the overhead regarding computational and communication costs for each client as well as GPU memories to reach $60\%$ accuracy for different methods on DomainNet. The results are presented in Table \ref{tab:overhead}.

\textbf{Computational costs:} Our proposed FedCiR spends less computational costs to reach $60\%$ accuracy against baselines.
Although FedCiR includes the feature generator and global representation distributions for local regularization, which increases the computational costs in each training round, the faster convergence rate significantly reduces the training rounds for FedCiR to reach $60\%$ accuracy, thus relieving the total computational costs for clients.

\textbf{Communication costs:} FedSR achieves the best performance and FedCiR realizes the second-best performance in communication cost saving.
This is because both of them achieve remarkably faster convergence rates compared with other methods, and additionally downloading the feature generator from the server takes FedCiR more communication costs than FedSR at clients.

\textbf{GPU memory:} Serving as a basic framework for FL, FedAvg occupies the least GPU memories for computation.
The other methods require more GPU memories for computation due to the extra components for local regularization, e.g., the global representations in FedProc and FedSR.
FedGen and FedCiR occupy the most GPU memories for computation, e.g., with 0.14GB more than FedAvg.
This is because they utilize a feature generator for regulating the local update and thus demand more GPU memories.
Even so, compared with the $7.71$GB GPU memories for FedAvg, the additional GPU memories (i.e., 0.14GB) for local regularization in FedCiR are small and negligible.

To sum up, our proposed FedCiR is able to save computational and communication costs compared with most of the baselines, which is attributed to the fast convergence rate. Moreover, the additional regularizers for solving the feature shift problem only increase negligible GPU memories over the baselines.

\begin{table}[h]
    \centering
    \resizebox{0.48\textwidth}{!}{
    \begin{tabular}{c||ccc}
    \hline
    \rowcolor[gray]{0.8}\multicolumn{4}{c}{\textbf{Overhead Analysis }} \\ \hline \hline
        {\multirow{2}{*}{{Algorithms}} } &  {\multirow{2}{*}{\makecell{Computation \\ (TFLOPs)}}} &   {\multirow{2}{*}{\makecell{Communication \\ (GB)}}} &  {\multirow{2}{*}{\makecell{GPU memory \\ (GB)}}} \\ \\ \hline \hline
         FedAvg & $84.08$ & $5.30$ & $\bm{7.71}$  \\
         FedProx & $84.77$ & $5.34$ & $7.82$ \\
         FedProc & $79.21$ & $4.99$ & \underline{$7.72$} \\
         FedGen & \underline{$49.37$} & $4.29$ & $7.85$ \\
         FedSR & $55.04$ & $\bm{3.28}$ & $7.80$ \\
         FedCiR & $\bm{42.31}$ &  \underline{$3.68$}  & $7.85$ \\ \hline
    \end{tabular}}
    \caption{Overhead analysis of varying algorithms for each client to reach $60\%$ accuracy in computation, communication, and GPU memory on DomainNet. The results in \textbf{bold}
indicate the best performance in overhead saving and the second-best results are \underline{underlined}.}
    \label{tab:overhead}
\end{table}

\section{Conclusions}
\label{conclusions}
In this paper, we proposed a novel federated representation learning framework, called FedCiR, to solve the feature-skewed issue in FL.
Based on our established generalization error bound for FL, we proposed to extract informative and client-invariant features by utilizing two mutual information terms.
FedCiR incorporated two regularization terms to bound these two mutual information terms with the global representation distribution, which not only improves global knowledge but also facilitates client-invariant representation training.
To mitigate the issue of the absence of global representation distribution in FL, we employed a generator to approximate it through a data-free mechanism.
Experiments verified the effectiveness of FedCiR in client-invariant representation extraction and solving the data heterogeneity issue.
In the future, we will investigate how to utilize pretrained generative models to solve the federated feature shift problem by generating the insufficient domains of data for devices.

\bibliographystyle{IEEEtran}
\bibliography{bibliography}

\begin{IEEEbiography}[{\includegraphics[width=1in,height=1.25in,clip,keepaspectratio]{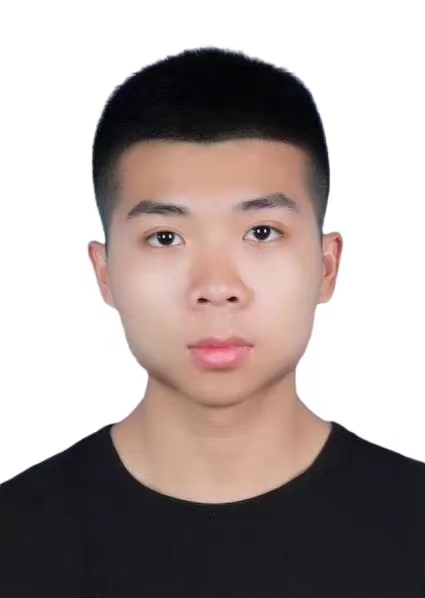}}]{Zijian Li}
(Graduate student member, IEEE) received the B.Eng. degree in Electrical Engineering and Automation from the South China University of Technology in 2020, and the M.Sc. degree in electronic and information engineering from the Hong Kong Polytechnic University in 2022.
He is currently pursuing a Ph.D. degree in the Department of Electronic and Computer Engineering at the Hong Kong University of Science and Technology.
His research interest is federated learning.
\end{IEEEbiography}
\vskip -2\baselineskip plus -1fil

\begin{IEEEbiography}[{\includegraphics[width=1in,height=1.25in,clip,keepaspectratio]{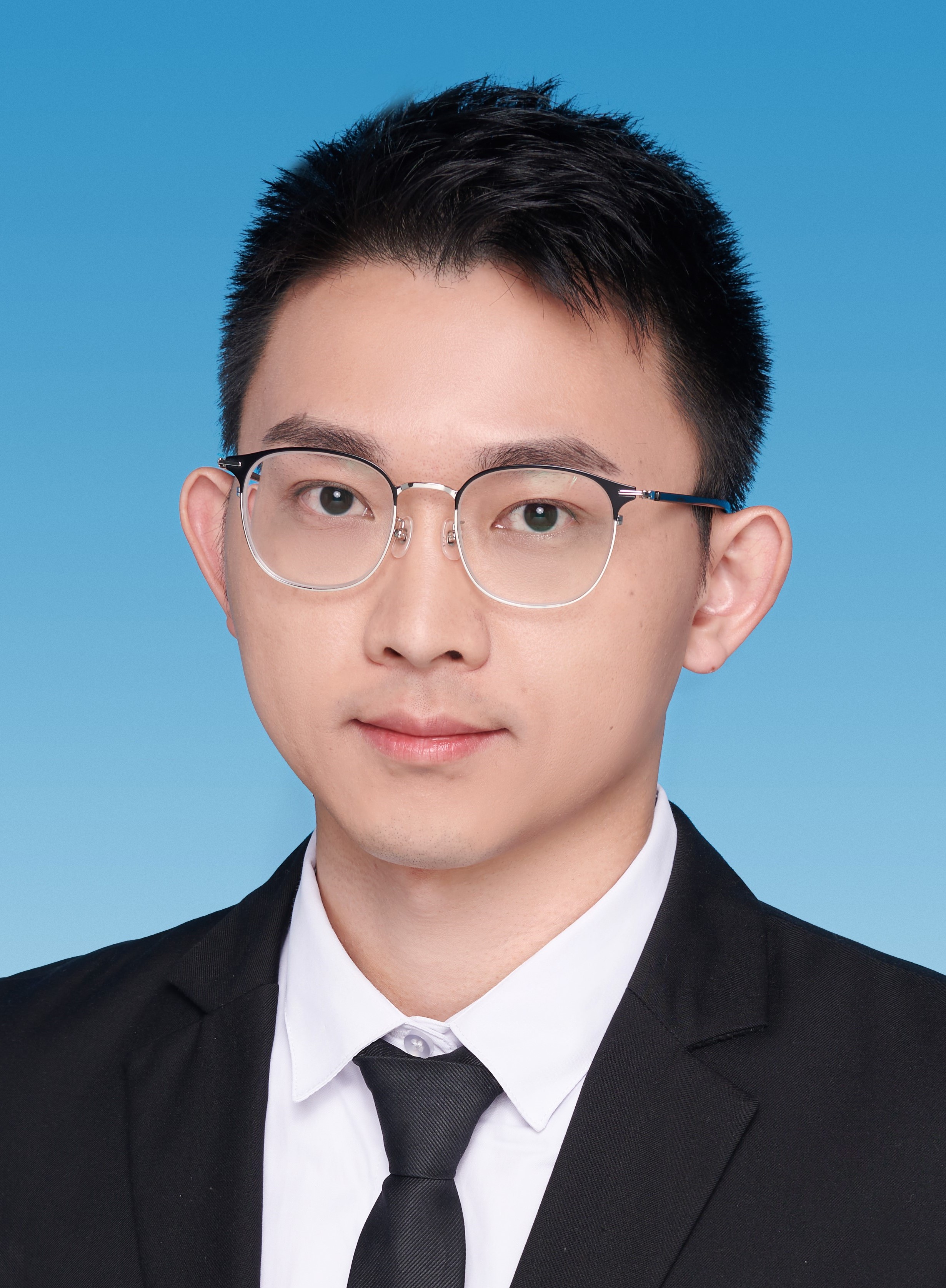}}]{Zehong Lin}
(Member, IEEE) received the B.Eng. degree in information engineering from South China University of Technology in 2017, and the Ph.D. degree in information engineering from The Chinese University of Hong Kong in 2022. Since 2022, he has been with the Department of Electronic and Computer Engineering, The Hong Kong University of Science and Technology, where he is currently a Research Assistant Professor. His research interests include federated learning and edge AI. 
\end{IEEEbiography}
\vskip -2\baselineskip plus -1fil

\begin{IEEEbiography}[{\includegraphics[width=1in,height=1.25in,clip,keepaspectratio]{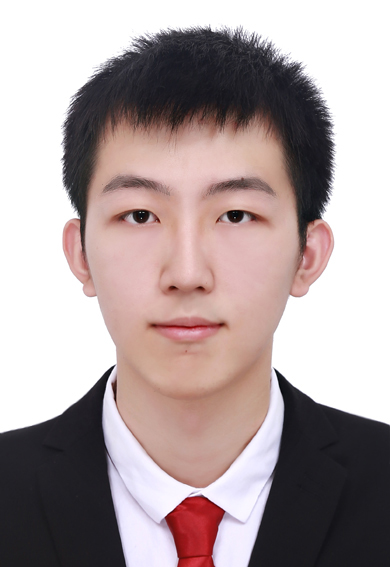}}]{Jiawei Shao}
(Graduate student member, IEEE) received the B.Eng. degree in telecommunication engineering from Beijing University of Posts and Telecommunications in 2019.
He is currently pursuing a Ph.D. degree in the Department of Electronic and Computer Engineering at the Hong Kong University of Science and Technology.
His research interests include edge intelligence and federated learning.
\end{IEEEbiography}
\vskip -2\baselineskip plus -1fil

\begin{IEEEbiography}[{\includegraphics[width=1in,height=1.25in,clip,keepaspectratio]{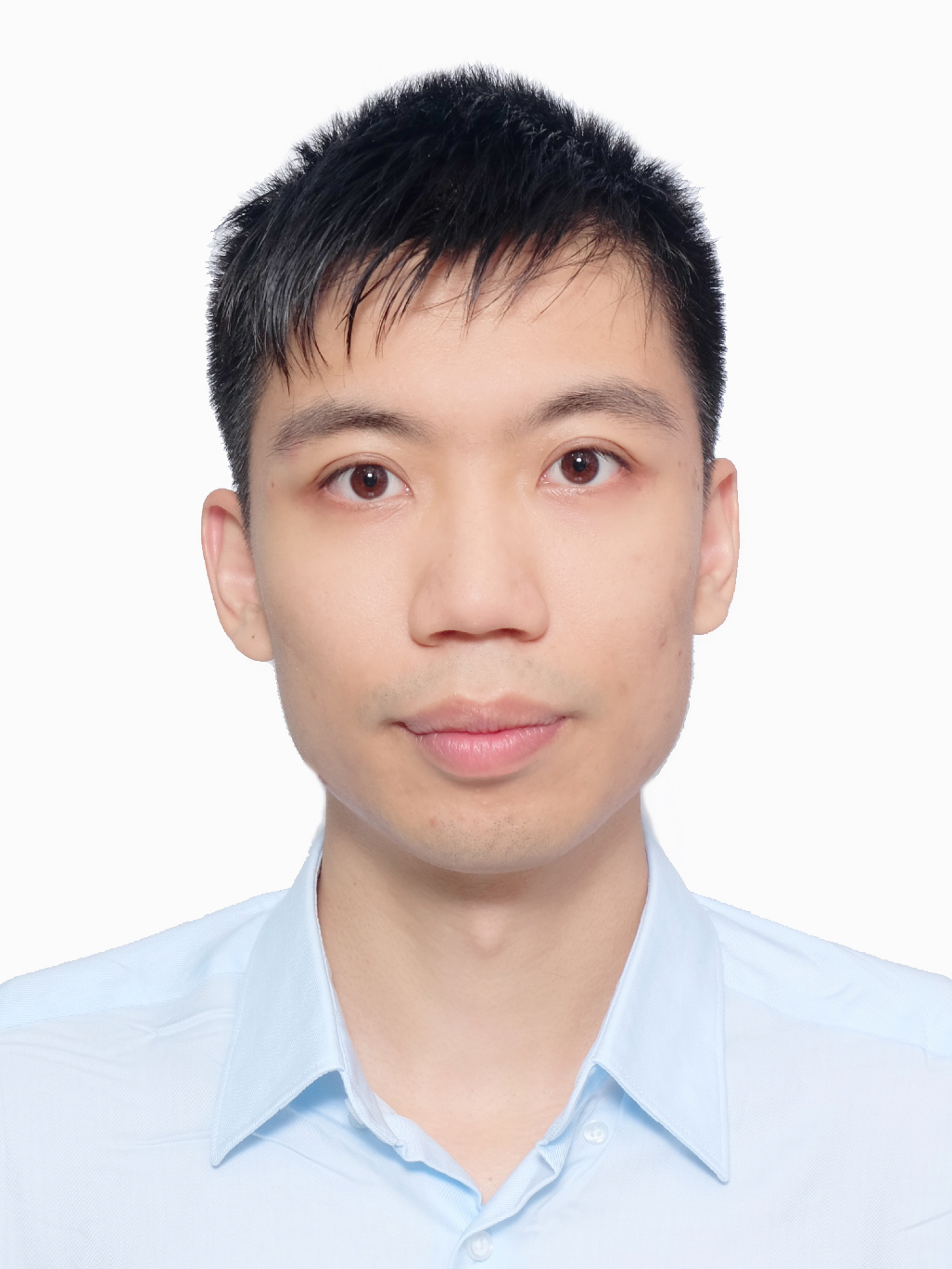}}]{Yuyi Mao}
(Member, IEEE) received the B.Eng. degree in information and communication engineering from Zhejiang University, Hangzhou, China, in 2013, and the Ph.D. degree in electronic and computer engineering from The Hong Kong University of Science and Technology, Hong Kong, in 2017. He was a Lead Engineer with the Hong Kong Applied Science and Technology Research Institute Co., Ltd., Hong Kong, and a Senior Researcher with the Theory Lab, 2012 Labs, Huawei Tech. Investment Co., Ltd., Hong Kong. He is currently a Research Assistant Professor with the Department of Electrical and Electronic, The Hong Kong Polytechnic University, Hong Kong. His research interests include wireless communications and networking, mobile-edge computing and learning, and wireless artificial intelligence.

He was the recipient of the 2021 IEEE Communications Society Best Survey Paper Award and the 2019 IEEE Communications Society and Information Theory Society Joint Paper Award. He was also recognized as an Exemplary Reviewer of the IEEE Wireless Communications Letters in 2021 and 2019 and the IEEE Transactions on Communications in 2020. He is an Associate Editor of the EURASIP Journal on Wireless Communications and Networking.
\end{IEEEbiography}
\vskip -2\baselineskip plus -1fil

\begin{IEEEbiography}[{\includegraphics[width=1in,height=1.25in,clip,keepaspectratio]{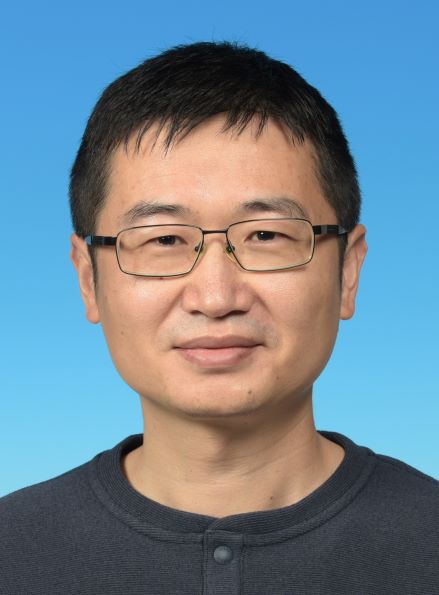}}]{Jun Zhang}

(Fellow, IEEE) received the B.Eng. degree in Electronic Engineering from the University of Science and Technology of China in 2004, the M.Phil. degree in Information Engineering from the Chinese University of Hong Kong in 2006, and the Ph.D. degree in Electrical and Computer Engineering from the University of Texas at Austin in 2009. He is an Associate Professor in the Department of Electronic and Computer Engineering at the Hong Kong University of Science and Technology. His research interests include wireless communications and networking, mobile edge computing and edge AI, and cooperative AI.

Dr. Zhang co-authored the book Fundamentals of LTE (Prentice-Hall, 2010). He is a co-recipient of several best paper awards, including the 2021 Best Survey Paper Award of the IEEE Communications Society, the 2019 IEEE Communications Society \& Information Theory Society Joint Paper Award, and the 2016 Marconi Prize Paper Award in Wireless Communications. Two papers he co-authored received the Young Author Best Paper Award of the IEEE Signal Processing Society in 2016 and 2018, respectively. He also received the 2016 IEEE ComSoc Asia-Pacific Best Young Researcher Award. He is an Editor of IEEE Transactions on Communications, IEEE Transactions on Machine Learning in Communications and Networking, and was an editor of IEEE Transactions on Wireless Communications (2015-2020). He served as a MAC track co-chair for IEEE Wireless Communications and Networking Conference (WCNC) 2011 and a co-chair for the Wireless Communications Symposium of IEEE International Conference on Communications (ICC) 2021. He is an IEEE Fellow and an IEEE ComSoc Distinguished Lecturer.
\end{IEEEbiography}
\vskip -2\baselineskip plus -1fil

\clearpage
\appendices

\section{Proof of Theorem \ref{theorem_2}}

\label{theorem_2_proof}
Consider the FL framework that collaboratively trains a global hypothesis $\hat{\bm{h}}$ to classify representations with $|\mathcal{K}|$ distributed source domains $\{\langle \mathcal{R}_{k}, \bm{r}_{k} \rangle \}_{k \in \mathcal{K}}$.
Since the global hypothesis is computed by a weighted sum for all the local hypotheses trained on local source domains, the generalization error bound can be computed as follows:
\begin{equation}
\label{fl_hypothesis}
    \varepsilon (\hat{\bm{h}}) 
     \equiv \varepsilon \left(\sum_{k \in \mathcal{K}} p(k) \hat{\bm{h}}_k\right).
\end{equation}
We use a one-layer fully-connected layer followed by a softmax layer as the hypothesis $\hat{\bm{h}}$ and the cross-entropy function to compute the risk $\varepsilon(\hat{\bm{h}})$, which is convex for any source data.
Combining with Theorem \ref{theorem_1}, we can bound \eqref{fl_hypothesis} as below:
\begin{align}
    \varepsilon (\hat{\bm{h}}) 
    \equiv& \varepsilon \left(\sum_{k \in \mathcal{K}} p(k) \hat{\bm{h}}_k \right)
    \leq \sum_{k \in \mathcal{K}} p(k) \varepsilon (\hat{\bm{h}}_k) \\
    \leq& \sum_{k \in \mathcal{K}} p(k) \varepsilon_{{k}} (\hat{\bm{h}}_k) + \sum_{k \in \mathcal{K}} p(k) d_{\hat{\mathcal{G}}} (\mathcal{R}_{k}, \Tilde{\mathcal{R}})  \notag \\
    &+ \sum_{k \in \mathcal{K}} p(k) \min \{ \mathbb{E}_{\mathcal{D}_{k}} [|\bm{r}_{k} - \Tilde{\bm{r}}|], \mathbb{E}_{\mathcal{D}_T} [|\bm{r}_{k} - \Tilde{\bm{r}}|] \}, \label{covariate}
\end{align}
where $\hat{\mathcal{G}}$ is the conditional function class for the hypotheses $\hat{\bm{g}}$ that $\hat{\bm{g}} (\bm{z}) \coloneqq \text{sign} (| \hat{\bm{h}} (\bm{z}) - \hat{\bm{h}}^\prime (\bm{z})| -m )$ with ${\hat{\bm{h}}}, {\hat{\bm{h}}}^\prime \in {\hat{\mathcal{H}}}$ and $0 \leq m \leq 1$, and $\mathcal{R}_{k}$ and $\Tilde{\mathcal{R}}$ are the induced representation distributions of $\mathcal{D}_{k}$ and $\Tilde{\mathcal{D}}$, respectively.

Under the {covariate shift} assumption that $\Tilde{\bm{r}}=\bm{r}_k$, the third term in (\ref{covariate}) equals 0. Hence, we obtain
\begin{align}
    &\varepsilon (\hat{\bm{h}}) 
    \leq \sum_{k \in \mathcal{K}} p(k) \varepsilon_{{k}} (\hat{\bm{h}}_k) + \sum_{k \in \mathcal{K}} p(k) d_{\hat{\mathcal{G}}} (\mathcal{R}_{k}, \Tilde{\mathcal{R}}).
\end{align}

\section{Proof of Proposition \ref{pro_3}}
\label{pro_3_proof}

By extending the mutual information term between labels and representations, we have
\begin{align}
    - I (Y;Z) &= - \int_Y \int_Z p(y,\bm{z}) \text{log} \frac{p(y,\bm{z})}{p(y) p(\bm{z})} d\bm{z} dy \\
     &= - \int_Y \int_Z p(y,\bm{z}) \text{log} p(y|\bm{z}) d\bm{z} dy + H(Y). \label{negative_mutual_information}
\end{align}
Notice that the entropy of labels $H(Y)$ is independent of model optimization. Thus, $H(Y)$ can be ignored in \eqref{negative_mutual_information} and we have
\begin{align}
     -I(Y;Z) &\equiv \int_Y \int_Z p(y,\bm{z}) [- \text{log} p(y|\bm{z})] d\bm{z} dy \\
     &= \sum_{k \in \mathcal{K}} p(k) \int_Y \int_Z p(y,\bm{z}) [- \text{log} {p}(y|\bm{z})] d\bm{z} dy.
\end{align}

Since the conditional distribution $p(y|\bm{z})$ is intractable in FL, we use the local predition distribution $\hat{p}(y|k,\bm{z})$ as a variational approximation to $p(y|\bm{z})$. According to the fact that the KL divergence is always positive, we have
\begin{align}
    &D_{\text{KL}} (p(Y|Z) || \hat{p}(Y|K,Z)) \geq 0 \\
    \Rightarrow &\int_Y \int_Z p(y,\bm{z}) \text{log} \frac{p(\bm{z}|y)}{\hat{p}(\bm{z}|k,y)} d\bm{z} dy \geq 0.
\end{align}
As the clients download the global model for local training every communication round, it is reasonable to use $\hat{p}(y|k,\bm{z})$ to approximate $p(y|\bm{z})$, and the distribution gap between them is significantly tight when the local steps are small.
Hence, we have
\begin{align}
    & \quad -I(Y;Z) \notag \\
    &\leq \sum_{k \in \mathcal{K}} p(k) \int_Y \int_Z  p(y,\bm{z}) [- \text{log} \hat{p}(y|k,\bm{z})] d\bm{z} dy \\
    &\leq \sum_{k \in \mathcal{K}} p(k) \int_Y p(y) \int_Z p(\bm{z}|y) [- \text{log} \hat{p}(y|k,\bm{z})] d\bm{z} dy.
\end{align}
By using $q(\bm{z}|y)$ to approximate the global representation distribution $p(\bm{z}|y)$, we have
\begin{align}
    & \quad -I(Y;Z) \notag \\
    &\leq \sum_{k \in \mathcal{K}} p(k) \int_Y p(y) \int_Z p(\bm{z}|y) [- \text{log} \hat{p}(y|k,\bm{z})] d\bm{z} dy \\
    &\approx \sum_{k \in \mathcal{K}} p(k) \int_Y p(y) \int_Z q(\bm{z}|y) [- \text{log} \hat{p}(y|k,\bm{z})] d\bm{z} dy \\
    &\leq \sum_{k \in \mathcal{K}} p(k) \mathbb{E}_{p(y)} \mathbb{E}_{q(\bm{z}|y)} [- \log \hat{p}(y|k,\bm{z})]
    =\mathcal{L}_{reg}.
\end{align}

\section{Proof of Proposition \ref{remark_1}}

\label{re_1_proof}
To prove that given a label class $y$, a representation $\bm{z}$ is client-invariant if and only if $I(\bm{z}; k | y)=0$, we establish the following separate proofs:
    \begin{itemize}
    \item Given a class $y$, if $I(\bm{z};k|y)=0$, then $p(\bm{z}|y)=p(\bm{z}|k,y)$, so that $p(\bm{z}|k,y)$ is invariant w.r.t. $k$.
    \item Given a class $y$, if $p(\bm{z}|k,y)$ is invariant w.r.t. $k$, then $\forall \bm{z}, k:$
    \begin{align}
        p(\bm{z}|y) &= \sum_{k \in \mathcal{K}} p(\bm{z}|k,y) p(k|y) \\
        &= \sum_{k \in \mathcal{K}} p(\bm{z}|k^\prime ,y) p(k|y) \\
        &= p(\bm{z}|k^\prime ,y) \sum_{k \in \mathcal{K}} p(k|y) \\
        &= p(\bm{z}|k^\prime ,y) \sum_{k \in \mathcal{K}}  \frac{p(k,y)}{p(y)} \\
        &= p(\bm{z}| k^\prime , y) \\
        &\Rightarrow I(\bm{z};k|y)=0.
    \end{align}
\end{itemize}

\section{Proof of Proposition \ref{proposition_1}}
\label{pro_4_proof}

By extending the information term between representations and clients conditioned on the labels, we have
    \begin{align}
        & \qquad  I (Z;K|Y) \notag\\
        &= \int_Y \sum_{k \in \mathcal{K}} \int_Z p(k,y,\bm{z}) \log \frac{p(k,\bm{z}|y)}{p(k|y)p(\bm{z}|y)} d\bm{z} dy \\
        &= \sum_{k \in \mathcal{K}}\int_Y \int_Z p(k,y) p(\bm{z}|k,y) \log \frac{p(\bm{z}|k,y)}{p(\bm{z}|y)} d\bm{z} dy.
    \end{align}
Similarly, due to the intractable $p(\bm{z}|y)$, we use $q(\bm{z}|y)$ to approximate it.
According to the fact that KL divergence is always positive, we have
\begin{align}
    & D_{\text{KL}} (p(Z|Y) || q(Z|Y)) \geq 0 \\
    \Rightarrow& \int_Y \int_Z p(\bm{z},y) \text{log} \frac{p(\bm{z}|y)}{q(\bm{z}|y)} d\bm{z} dy \geq 0 \\
    \Rightarrow & \sum_{k \in \mathcal{K}} \int_Y\int_Z p(k,y,\bm{z}) \log \frac{p(\bm{z}|y)}{q(\bm{z}|y)} d\bm{z} dy \geq 0 \\ 
    \Rightarrow & \sum_{k \in \mathcal{K}} \int_Y \int_Z p(k,y,\bm{z})  \log p(\bm{z}|y) d\bm{z} dy \notag\\
    & \quad \quad\geq  \sum_{k \in \mathcal{K}} \int_Y \int_Z p(k,y,\bm{z}) \text{log} q(\bm{z}|y) d\bm{z} dy. \label{39}
\end{align}
Then, we have
\begin{align}
    & \qquad I (Z,K|Y) \notag \\
    &\leq \sum_{k \in \mathcal{K}} \int_Y \int_Z p(k,y) p(\bm{z}|k,y) \log \frac{p(\bm{z}|k,y)}{q(\bm{z}|y)} d\bm{z} dy \\
    &= \sum_{k \in \mathcal{K}} \int_Y p(k,y) \int_Z p(\bm{z}|k,y) \log \frac{p(\bm{z}|k,y)}{q(\bm{z}|y)} d\bm{z} dy.
\end{align}
Let $m=p(\bm{z}|k,y,\bm{x})$ and $f(m)=m \log (m) (m>0)$.
Since $f(m)$ is a convex function, according to the Jensen's Inequality, we obtain:
\begin{align}
    &p(\bm{z}|k,y) \log p(\bm{z}|k,y) \notag \\
    \leq \int_X& p(\bm{x}|k,y) p(\bm{z}|k,y,\bm{x}) \log p(\bm{z}|k,y,\bm{x}).
\end{align}
By including it into equation \eqref{39}, we get
\begin{align}
    & \qquad I (Z,K|Y) \notag \\
    &\leq \sum_{k \in \mathcal{K}} \int_X \int_Y \int_Z p(k,\bm{x},y,\bm{z}) \log  \frac{p(\bm{z}|k,y,\bm{x})}{q(\bm{z}|y)} d\bm{z} dy d\bm{x} \\
    &\leq  \sum_{k \in \mathcal{K}} p(k) \mathbb{E}_{p(\bm{x},y|k)}  D_{\text{KL}} \big(p(\bm{z}|k,\bm{x}) \| q(\bm{z}|y) \big) \\
    &= \mathcal{L}_{align}.
\end{align}

\section{Hyperparameters}
\label{hyper}
In the experiments, we tune the hyperparameters for our proposed method and the baselines and report the best results.
Specifically, we tune the hyperparameters for FedProx ($\mu$) over $\{0.1,0.5,0.9\}$. For FedGen, we tune the hyperparameters ($\lambda_{kd}, \lambda_{reg}$) over $\{(0.5,1.0), (1.0,0.5), (1.0,1.0)\}$. For FedSR, we search the hyperparameters ($\alpha^{L2R},\alpha^{CMI}$) over $\{(5\times 10^{-7},0), (6\times 10^{-1},0), (5\times 10^{-5},0), (1\times 10^{-6},1\times 10^{-2}),(1\times 10^{-6},1\times 10^{-3})\}$.
For FedCiR, we tune the hyperparameters ($\lambda_{reg},\lambda_{align}$) among $\{ (5\times 10^{-7},0.5), (1\times 10^{-6},0.5), (5\times 10^{-5},0.5) \}$.
The hyperparameters for the experiments in Table \ref{main_results} are shown as follows:

\textbf{DomainNet:}
\begin{itemize}
    \item FedProx: $\mu=0.1$
    \item FedGen: $\lambda_{kd}=1.0$, $\lambda_{reg}=0.5$
    \item FedSR: $\alpha^{L2R}=0$, $ \alpha^{CMI}=5\times 10^{-7}$
    \item FedReg: $\lambda_{align}=0$, $ \lambda_{reg}=0.5$
    \item FedAlign: $\lambda_{align}=1\times 10^{-6}$, $ \lambda_{reg}=0$
    \item FedCiR: $\lambda_{align}=1\times 10^{-6}$, $ \lambda_{reg}=0.5$
\end{itemize}

\textbf{PACS:}
\begin{itemize}
    \item FedProx: $\mu=0.1$
    \item FedGen: $\lambda_{kd}=1.0$, $ \lambda_{reg}=0.5$
    \item FedSR: $\alpha^{L2R}=0$, $ \alpha^{CMI}=1\times 10^{-6}$
    \item FedReg: $\lambda_{align}=0$, $ \lambda_{reg}=0.5$
    \item FedAlign: $\lambda_{align}=5\times 10^{-7}$, $ \lambda_{reg}=0$
    \item FedCiR: $\lambda_{align}=5\times 10^{-7}$, $ \lambda_{reg}=0.5$
\end{itemize}

\textbf{Office-Caltech-10:}
\begin{itemize}
    \item FedProx: $\mu=0.5$
    \item FedGen: $\lambda_{kd}=0.5$, $ \lambda_{reg}=1.0$
    \item FedSR: $\alpha^{L2R}=0$, $ \alpha^{CMI}=1\times 10^{-6}$
    \item FedReg: $\lambda_{align}=0$, $ \lambda_{reg}=0.5$
    \item FedAlign: $\lambda_{align}=5\times 10^{-7}$, $ \lambda_{reg}=0$
    \item FedCiR: $\lambda_{align}=5\times 10^{-7}$, $ \lambda_{reg}=0.5$
\end{itemize}

\textbf{Camelyon17:}
\begin{itemize}
    \item FedProx: $\mu=0.9$
    \item FedGen: $\lambda_{kd}=1.0$, $ \lambda_{reg}=1.0$
    \item FedSR: $\alpha^{L2R}=0$, $ \alpha^{CMI}=5\times 10^{-7}$
    \item FedReg: $\lambda_{align}=0$, $ \lambda_{reg}=0.5$
    \item FedAlign: $\lambda_{align}=5\times 10^{-7}$, $ \lambda_{reg}=0$
    \item FedCiR: $\lambda_{align}=5\times 10^{-7}$, $ \lambda_{reg}=0.5$
\end{itemize}

\textbf{Iwildcam:}
\begin{itemize}
    \item FedProx: $\mu=0.1$
    \item FedGen: $\lambda_{kd}=1.0$, $ \lambda_{reg}=1.0$
    \item FedSR: $\alpha^{L2R}=0$, $ \alpha^{CMI}=5\times 10^{-7}$
    \item FedReg: $\lambda_{align}=0$, $ \lambda_{reg}=0.5$
    \item FedAlign: $\lambda_{align}=1\times 10^{-6}$, $ \lambda_{reg}=0$
    \item FedCiR: $\lambda_{align}=1\times 10^{-6}$, $ \lambda_{reg}=0.5$
\end{itemize}

\end{document}